\documentclass{article}

\usepackage{arxiv}
\usepackage{musicography}
\usepackage[utf8]{inputenc} 
\usepackage[T1]{fontenc}    
\usepackage{hyperref}       
\usepackage{url}            
\usepackage{booktabs}       
\usepackage{amsfonts}       
\usepackage{nicefrac}       
\usepackage{microtype}      
\usepackage{graphicx}
\usepackage{natbib}
\usepackage{doi}
\usepackage{algorithm}
\usepackage{algpseudocode}
\usepackage{float}
\usepackage{tikz}
\usepackage{url}

\usepackage{booktabs}  
\usepackage{subcaption}
\usepackage{amsmath}
\usepackage{amssymb}
\usepackage{amsthm}

\theoremstyle{plain}
\newtheorem{theorem}{Theorem}[section]

\newtheorem{lemma}[theorem]{Lemma}

\theoremstyle{definition}
\newtheorem{definition}[theorem]{Definition}
\newtheorem{assumption}[theorem]{Assumption}
\newtheorem{example}[theorem]{Example}

\theoremstyle{remark}
\newtheorem{remark}[theorem]{Remark}

\DeclareMathOperator{\id}{id}

\newcommand{\R}{\mathbb{R}}

\title{GeoHNNs: Geometric Hamiltonian Neural Networks}

\author{
Amine Mohamed Aboussalah\thanks{Equal Contribution.} \\
NYU Tandon School of Engineering \\
\texttt{ama10288@nyu.edu}
    \And
Abdessalam Ed-dib$^*$
\\
NYU Tandon School of Engineering 
\\
\texttt{ae2842@nyu.edu}
}

\makeatother

\renewcommand{\headeright}{A preprint}
\renewcommand{\shorttitle}{GeoHNN: Geometric Hamiltonian Neural Networks}

\hypersetup{
pdftitle={GeoHNN: Geometric Hamiltonian Neural Networks},
pdfsubject={cs.AI, cs.LG, stat.ML, math.GT},
pdfauthor={Amine M. Aboussalah, Abdessalam Ed-dib},
pdfkeywords={Hamiltonian Neural Networks, Physics Inspired Machine Learning, Dynamical Systems, Information Geometry, Riemannian Geometry},
}

\begin{document}
\maketitle
\begin{abstract}
The fundamental laws of physics are intrinsically geometric, dictating the evolution of systems through principles of symmetry and conservation. While modern machine learning offers powerful tools for modeling complex dynamics from data, common methods often ignore this underlying geometric fabric. Physics-informed neural networks, for instance, can violate fundamental physical principles, leading to predictions that are unstable over long periods, particularly for high-dimensional and chaotic systems. Here, we introduce \textit{Geometric Hamiltonian Neural Networks (GeoHNN)}, a framework that learns dynamics by explicitly encoding the geometric priors inherent to physical laws. Our approach enforces two fundamental structures: the Riemannian geometry of inertia, by parameterizing inertia matrices in their natural mathematical space of symmetric positive-definite matrices, and the symplectic geometry of phase space, using a constrained autoencoder to ensure the preservation of phase space volume in a reduced latent space. We demonstrate through experiments on systems ranging from coupled oscillators to high-dimensional deformable objects that GeoHNN significantly outperforms existing models. It achieves superior long-term stability, accuracy, and energy conservation, confirming that embedding the geometry of physics is not just a theoretical appeal but a practical necessity for creating robust and generalizable models of the physical world.
\end{abstract}
\section{Introduction}

Many physical systems are described by differential equations whose solutions obey conservation laws arising from underlying symmetries. According to Noether’s theorem \cite{Noether1918}, every continuous symmetry of a physical system corresponds to a conserved quantity. For instance, time invariance implies conservation of energy, whereas spatial and rotational invariance correspond to conservation of linear and angular momentum, respectively. These laws can be interpreted as geometric priors, reflecting invariances in space and time that govern the dynamics of the system.

In Hamiltonian mechanics \cite{goldstein:mechanics}, the state of a physical system is described by a point in the phase space. This phase space encodes the system’s instantaneous position and momentum, and trajectories in this space represent the temporal evolution of the dynamical system. For a closed dynamical system, this evolution follows specific rules, called the Hamilton’s equations, that ensure the dynamics of the system preserves the conservation of energy. Indeed, the evolution of the system’s dynamics does not occur arbitrarily within the phase space but is constrained by an underlying geometry called “symplectic geometry.” In simple terms, a symplectic space is a space equipped with a symplectic form, which provides a way to measure oriented areas spanned by pairs of vectors \cite{arnold1989mathematical}. One key consequence of this geometry is expressed by Liouville’s theorem \cite{Liouville1838}, which asserts that the total volume occupied by a group of nearby states\footnote{The state of a Hamiltonian system is given by a pair consisting of its position and momentum.} of a Hamiltonian system in phase space remains constant over time. Put simply, as the system evolves, the “spread” or uncertainty of possible states does not shrink or expand, that is the phase space volume is conserved.

That said, capturing these geometry-driven dynamics directly from data is difficult. Models, such as multilayer perceptrons \cite{lecun2015deep}, typically lack the inductive biases required to capture the geometric constraints of the underlying system \cite{need_for_biases, inductive_biases}. Therefore, they often violate basic physical principles, leading to unrealistic predictions, poor long-term stability, and limited generalization capabilities. 

To address this, recent approaches have embedded physical laws directly into machine learning models \cite{governingequationsdiscovery, hamiltonian_neural_networks, deepleangragiaannetworks, lnns,  symplecticrecurrentneuralnetworks, simple_hnn_lnn, generalizedhamiltoninas}. In particular, HNNs \cite{hamiltonian_neural_networks} have shown success by parameterizing the Hamiltonian function as a neural network and deriving dynamics via Hamilton’s equations encoded as a cost function. These methods enforce conservation of energy leading to improved performance on tasks such as trajectory prediction in simple systems, like a mass spring or a single pendulum. However, the effectiveness of HNNs is considerably reduced in complex dynamical systems. For instance, in high-dimensional or chaotic systems, such as the three-body problem, HNNs often exhibit energy drift, degraded long-term prediction accuracy, and limited generalization.

These shortcomings arise from several factors: (i) the difficulty of learning accurate Hamiltonians in high-dimensional phase spaces; (ii) the accumulation of numerical errors due to approximate integration schemes; and (iii) the inability to fully preserve key geometric properties, such as the symplecticity of the phase space.

A promising approach to handle the complexity of high-dimensional Hamiltonian systems is model order reduction (MOR) \cite{model_order_reduction, buchfink2024modelreductionmanifoldsdifferential}, which approximates the original dynamics using lower-dimensional representation, also called a “latent space.” The main goal is to find a latent space where the dynamics of the system is more tractable, learn the reduced dynamics in that latent space, and then lift (reconstruct) the solution back to the original phase space.

Traditional MOR methods, like Proper Orthogonal Decomposition (POD) \cite{berkooz_pod}, use linear projections to reduce dimensionality by capturing the most important features of the data. In our context, the “data” refers to the time evolution of the system’s state, specifically its positions and momenta over time. However, POD does not guarantee the preservation of the symplectic geometry that is essential for correctly modeling Hamiltonian dynamics. Thus, the reduced models may lose important physical properties, also referred to as features of the dynamical system in machine learning. To overcome this, methods such as Proper Symplectic Decomposition (PSD) \cite{peng2015symplecticmodelreductionhamiltonian} have been developed to ensure that the reduced representation maintains the symplectic form. But PSD still relies on linear assumptions, which limits its ability to capture more complex, nonlinear system behavior, such as robot arm manipulation with flexible joints, and legged locomotion involving dynamic contact and balance. More recently, nonlinear techniques based on neural networks, such as autoencoders \cite{bank2021autoencoders, côte2024hamiltonianreductionusingconvolutional, venkat2021convolutionalautoencodersreducedordermodeling}, have been used to learn more flexible latent spaces. Although these methods can capture nonlinear features, vanilla autoencoders do not necessarily preserve the geometric constraints such as the symplecticity of the latent space. Without these constraints, the reduced models do not guarantee the preservation of symmetries of the system, leading potentially to inaccurate or unstable predictions.

In this work, we introduce a principled framework for learning reduced-order Hamiltonian dynamics that explicitly respects the underlying geometry of physical systems. To the best of our knowledge, this framework is the first to learn reduced Hamiltonian dynamics by jointly incorporating two complementary geometric priors: (i) the symplectic geometry of phase space, which we preserve by using a constrained symplectic autoencoder designed to enforce symplecticity in the latent space , and (ii) the Riemannian geometry underlying physical quantities, such as inertia matrices common in rigid body systems found in robotics, that lie on the space of symmetric positive definite (SPD) matrices \cite{goldstein:mechanics}. Our key contributions are threefold: 
\begin{enumerate}
    \item Formulating a theoretical framework that rigorously formalizes how incorporating geometric constraints improves the learning and modeling of reduced Hamiltonian dynamics.
    \item Developing a neural network model that effectively incorporates symplectic and Riemannian geometric priors by combining a constrained symplectic autoencoder with SPD-constrained parameterizations of inertia matrices to learn reduced Hamiltonian dynamics.
    \item Conducting comprehensive experiments that demonstrate better generalization, improved accuracy, long-term stability, and adherence to physical laws such as energy conservation across a variety of physical systems with differing complexity, degrees of freedom, and dynamics (both linear and nonlinear). 
\end{enumerate}

\section{Background}
In this section, we briefly review the mathematical background and prior work relevant to our approach. We begin with key concepts from differential geometry, which provide the foundation for modeling physical systems with geometric structure. By geometric structure, we mean the set of symmetries (invariances) and constraints that govern the evolution of these systems. We then review the Lagrangian and Hamiltonian formulations of classical mechanics, which formalize the dynamics of conservative systems. For example, these approaches can be used to model the movement of a robotic arm to understand how its joints and links move and interact. Finally, we discuss MOR techniques and their challenges in preserving the geometry of physical systems.

\subsection{An Overview of Differential Geometry}
Differential geometry \cite{Isham1989ModernDG,lee_smooth_manifolds} provides the mathematical framework for describing the geometric structures that govern physical systems.
Many physical systems do not evolve in flat Euclidean space, but instead on more general spaces called \emph{manifolds}.
For example, the possible positions of a robotic arm with revolute joints lie on the torus $\mathbb{T}^n = (S^1)^n$, where $n$ is the number of joints. 
More generally, the set of possible positions and orientations of a rigid body in 3D space is the Lie group $\mathrm{SE}(3)$, which accounts for both translation and rotation. 
Both the torus $\mathbb{T}^n$ and $\mathrm{SE}(3)$ are examples of manifolds.  These examples illustrate that accurately capturing the geometry and constraints of real-world systems requires modeling dynamics on appropriate manifolds.

A manifold $\mathcal{M}$ is a topological space that is locally Euclidean: around every point $p \in \mathcal{M}$, there exists a neighborhood that can be mapped homeomorphically to an open subset of $\mathbb{R}^n$.
These mappings are called \emph{charts}, and a collection of charts that covers the manifold forms an \emph{atlas}.
If the coordinate transitions between overlapping charts are smooth (i.e., differentiable as many times as needed), the manifold is called a \emph{smooth manifold}. 

This smooth structure allows us to perform differential calculus on $\mathcal{M}$, such as defining tangent vectors (velocities), vector fields (e.g., fluid velocity, force, acceleration, or control fields), and differential equations on manifolds~\cite{tu2020differential}. At each point $p \in \mathcal{M}$, the \emph{tangent space} $T_p\mathcal{M}$ is the vector space consisting of the tangent vectors to all smooth curves on $\mathcal{M}$ that pass through $p$. 
Intuitively, it represents the space of all possible velocity directions at $p$. The disjoint union of all tangent spaces defines the \emph{tangent bundle} $T\mathcal{M}$, which itself forms a smooth manifold of dimension $2n$ when $\mathcal{M}$ has dimension $n$, since each point in \( T\mathcal{M} \) pairs an \( n \)-dimensional position with an \( n \)-dimensional velocity. Moreover, each of these tangent spaces admits a natural dual space, which is called the \emph{cotangent space} $T^*_p\mathcal{M}$. It represents the space of linear maps from $T_p\mathcal{M}$ to $\mathbb{R}$ (i.e. linear forms). The \emph{cotangent bundle} $T^*\mathcal{M}$ is the disjoint union of these cotangent spaces, which is also a manifold of dimension $2n$.

To describe dynamics on a manifold, we use a \emph{vector field} $X : \mathcal{M} \to T\mathcal{M}$, which maps each point on the manifold to one of its tangent vectors. The space of all smooth vector fields is denoted $\mathfrak{X}(\mathcal{M})$, and an initial value problem (IVP) can be formulated: the integral curves $\gamma(t)$ of $X$ satisfy the differential equation
\[
\dot{\gamma}(t) = X(\gamma(t)),
\]
meaning that the curve's velocity at each point agrees with the assigned vector field. These integral curves describe the trajectories of physical systems evolving on $\mathcal{M}$, and can be written in integral form as follows:

\[
\gamma(t) = \gamma(0) + \int_0^t X(\gamma(s))\, ds.
\]

For more rigorous definitions and formal properties, please refer to the appendix.

\subsection{Configuration Manifold and Lagrangian Formulation}
In this work, we adopt the notation $\mathcal{Q}$ instead of $\mathcal{M}$ to denote the configuration manifold, which is standard in mechanics. 
The configuration manifold $\mathcal{Q}$ is an $n$-dimensional smooth manifold parameterized by the generalized coordinates of the physical system. 
Generalized coordinates are variables that uniquely define the system's configuration. 
A configuration refers to the set of degrees of freedom required to specify its position and orientation completely. For example, the angle of a single pendulum relative to the vertical axis is a generalized coordinate. 

More generally, each point $q \in \mathcal{Q}$ corresponds to a unique configuration, expressed as $q = (q_1, q_2, \dots, q_n) \in \mathbb{R}^n$. The configuration alone cannot describe the evolution of the system; velocity is also needed. 
This is because the future behavior depends on both where the system is (position $q$) and how it is moving at that moment (velocity $\dot{q}$). Mathematically, the state of the system is represented as a point in the tangent bundle $T\mathcal{Q}$, where each element $(q, \dot{q}) \in T\mathcal{Q}$ represents a position-velocity pair, which we refer to as the ``state'' of the system. 

The dynamics of the system is then derived from the Lagrangian $L : T\mathcal{Q} \rightarrow \mathbb{R}$, typically defined on $T\mathcal{Q}$ as the difference between kinetic and potential energy:
\begin{equation}
    \label{eq: lagrangian_form}
    L(q, \dot{q}) = T(q, \dot{q}) - V(q),
\end{equation}

where the kinetic energy is $T(q, \dot{q}) = \frac{1}{2} \dot{q}^\top M(q) \dot{q}$, with $M(q)$ the inertia matrix, and $V(q)$ the potential energy.

To obtain the equations of motion, we apply the ``principle of least action'' \cite{Lanczos1970}, a variational principle stating that the physical trajectory $\gamma$ of the system extremizes the action functional, defined as the time integral of the Lagrangian between two time steps $t_0$ and $t_f$:
\[
\mathcal{S}[\gamma] = \int_{t_0}^{t_f} L(\gamma(t), \dot{\gamma}(t)) \, dt.
\]
Formally, the system's trajectory $\gamma : [t_0, t_f] \to \mathcal{Q}$ is obtained by solving the optimization problem
\[
\min_{\gamma} \mathcal{S}[\gamma] = \min_{\gamma} \int_{t_0}^{t_f} L(\gamma(t), \dot{\gamma}(t)) \, dt,
\]
subject to suitable boundary conditions on $\gamma(t_0)$ and $\gamma(t_f)$.

This leads to the Euler--Lagrange equations \cite{goldstein:mechanics}, which describe how a point $(q, \dot{q})$ on $T\mathcal{Q}$ evolve over time in the absence of external forces:
\[
\frac{d}{dt} \frac{\partial L}{\partial \dot{q}} - \frac{\partial L}{\partial q} = 0.
\]
In the presence of external generalized forces $F(q, \dot{q}) \in \mathbb{R}^n$ acting on the system, the motion is governed by the forced Euler--Lagrange equations:
\[
\frac{d}{dt} \frac{\partial L}{\partial \dot{q}} - \frac{\partial L}{\partial q} = F(q, \dot{q}).
\]
Expanding this, one obtains a second-order differential equation expressing acceleration explicitly:
\[
M(q) \ddot{q} + C(q, \dot{q}) \dot{q} + \nabla V(q) = F(q, \dot{q}),
\]
where:
\begin{itemize}
    \item $\ddot{q}$ is the acceleration (second derivative of generalized coordinates),
    \item $C(q, \dot{q})$ contains Coriolis and centrifugal terms arising from the inertia matrix,
    \item $\nabla V(q)$ is the gradient of the potential energy with respect to the generalized coordinates.
\end{itemize}

Recent studies have leveraged this Lagrangian formulation as an inductive bias for learning physically plausible dynamics.  Deep Lagrangian Networks \cite{deepleangragiaannetworks} incorporate the Euler--Lagrange equations into the learning objective (cost function) by minimizing the mean squared error between the predicted and ground truth generalized forces. To achieve this, the total force is decomposed into physically interpretable components: inertial, Coriolis, and gravitational terms, each parameterized by neural networks. However, their formulation assumes a quadratic kinetic energy, which restricts applicability to rigid-body systems, such as those commonly found in robotics, and many-particle systems like those studied in molecular dynamics~\cite{robotics_systems, molecular_dynamics}. Lagrangian Neural Networks (LNNs) \cite{lnns} generalize this idea by directly parameterizing the Lagrangian with a neural network and using automatic differentiation to obtain the equations of motion. This allows modeling systems with more complex dynamics, such as those involving relativistic corrections or electromagnetic interactions~\cite{lnns}. Generalized LNNs \cite{xiao2024generalizedlagrangianneuralnetworks} further extend this framework by explicitly modeling non-conservative forces, enabling the learning of dissipative dynamics.

Despite their promise, learning dynamics through the Lagrangian framework introduces several challenges:
\begin{itemize}
    \item \textbf{Second-order differentiation:} Learning dynamics via the Euler--Lagrange equations requires computing accelerations $\ddot{q}$, which involves second-order automatic differentiation with respect to both inputs and network parameters. This significantly increases computational cost and memory usage, particularly in high-dimensional systems. To address this, some works leverage Neural ODEs \cite{chen2019neuralordinarydifferentialequations, zhong2024symplecticodenetlearninghamiltonian} to integrate second-order dynamics and compute trajectory-level losses using positions and velocities.
    \item \textbf{Inertia matrix degeneracy:} The inertia matrix $M(q)$ may become singular or ill-conditioned at certain configurations, leading to numerical instabilities in the system of equations:
\[
\ddot{q} = M(q)^{-1} \bigl( F(q, \dot{q}) - C(q, \dot{q}) \dot{q} - \nabla V(q) \bigr).
\]
\end{itemize}

Such degeneracies can cause exploding or vanishing gradients during training and compromise the stability and accuracy of the learned dynamics.

These challenges motivate the adoption of the Hamiltonian formulation, which expresses dynamics as a first-order system on the cotangent bundle $T^*\mathcal{Q}$ using position and momentum variables. This approach encodes the symplectic geometry of phase space via the structure of Hamilton’s equations, which depend only on first-order derivatives of the Hamiltonian. \cite{goldstein:mechanics}. Moreover, computing the Hamiltonian function only requires access to the inverse of the inertia matrix, in contrast to the Lagrangian framework, which involves computing both the inertia matrix and its inverse. Thus, we can directly parameterize the inverse inertia matrix using a neural network, eliminating the need for any explicit inversion and reducing potential numerical instabilities associated with it.

\subsection{Hamiltonian Formulation on the Cotangent Bundle}
To address the aforementioned limitations, we adopt the Hamiltonian formalism.  
To transition to the Hamiltonian setting, we apply the \emph{Legendre transform} \cite{arnold1989mathematical}, which reformulates the dynamics in terms of momenta rather than velocities.  
The Hamiltonian function $H: T^*\mathcal{Q} \rightarrow \mathbb{R}$ is defined as the Legendre dual of the Lagrangian $L$:
\begin{equation}
    H(q, p) = \sup_{\dot{q} \in T_q\mathcal{Q}} \left( \langle p, \dot{q} \rangle - L(q, \dot{q}) \right),
\end{equation}
where $\langle \cdot, \cdot \rangle$ denotes the natural pairing between the cotangent and tangent spaces, i.e., $p \in T_q^*\mathcal{Q}$ and $\dot{q} \in T_q\mathcal{Q}$.

If the Lagrangian is \emph{regular}, meaning the Legendre map
\begin{equation}
    \dot{q} \mapsto \frac{\partial L(q, \dot{q})}{\partial \dot{q}}
\end{equation}
is a local diffeomorphism, we can explicitly solve for the conjugate momentum:
\begin{equation}
    p = \frac{\partial L(q, \dot{q})}{\partial \dot{q}}.
\end{equation}

For many mechanical systems, such as pendulums, robotic arms, or satellites in orbit, the Lagrangian takes form (\ref{eq: lagrangian_form}). In this case, the Legendre transform simplifies to:
\begin{equation}
    p = M(q) \dot{q},
\end{equation}
and the corresponding Hamiltonian becomes:
\begin{equation}
    H(q, p) = \frac{1}{2} p^\top M(q)^{-1} p + V(q),
\end{equation}
which is the total energy (kinetic + potential) of the system.

In closed systems, this energy is conserved over time.  
To encode this conservation law, the cotangent bundle $T^*\mathcal{Q}$, also called the \emph{phase space}, is equipped with a symplectic form, denoted by $\omega$.  
This symplectic form is a closed, nondegenerate 2-form, which in canonical coordinates (position and momentum) can be written as
\[
\omega = dq \wedge dp.
\]
It defines the geometry that governs the dynamics of the system, ensuring both the preservation of phase space volume, as stated by Liouville’s theorem, and the conservation of energy, which follows from Noether’s theorem relating symmetries to conserved quantities throughout the system’s evolution. Indeed, given the Hamiltonian function $H$, the evolution of the system is governed by the Hamiltonian vector field $X_H$, according to the following IVP:

\begin{equation}
\label{eq: full-order-model}
    \begin{cases}
        \dot{z}(t) = X_H(z(t)), \\
        z(0) = z_0 \in T^*\mathcal{Q},
    \end{cases}
\end{equation}
where $z = (q, p)$.

Expressed in the canonical coordinates $(q, p)$, this leads to Hamilton’s equations:
\begin{equation}
    \label{eq: hamilton_equations}
    \dot{q} = \frac{\partial H}{\partial p}, \quad \dot{p} = - \frac{\partial H}{\partial q}.
\end{equation}

Under standard smoothness and regularity assumptions on the Hamiltonian function $H$, the Picard–Lindelöf theorem guarantees the existence and uniqueness of local solutions\footnote{A local solution is one that exists and is unique only on a short time interval around the initial condition.} \cite{walter1998ordinary}. To approximate these solutions from data, HNNs parameterize the Hamiltonian function $H_\theta(q,p)$ using a neural network.  
This neural network enforces the inductive bias of Hamiltonian mechanics by using Hamilton’s equations to obtain the time derivatives $\dot{q}$ and $\dot{p}$. Instead of regressing these derivatives directly as in generic machine learning models, such as multilayer perceptrons, HNNs first evaluate the Hamiltonian function $H_\theta(q,p)$, then compute its partial derivatives with respect to the inputs $q$ and $p$, and finally apply Hamilton’s equations (\ref{eq: hamilton_equations}) to obtain the temporal derivatives. These are then compared to the ground truth derivatives using a loss function of our choice.

However, computing the loss function does not necessarily require explicitly using the derivatives. Instead of matching derivatives directly, models like Symplectic Recurrent Neural Networks and Symplectic ODE-Nets simulate the system forward in time by integrating the predicted temporal derivatives \cite{symplecticrecurrentneuralnetworks, zhong2024symplecticodenetlearninghamiltonian}. This integration is performed using symplectic integrators, numerical methods specifically designed to preserve the symplectic structure of Hamiltonian systems during time evolution, resulting in predicted trajectories. These predicted trajectories are then compared to the ground truth trajectories using a loss function, allowing the model to learn the dynamics of the physical systems.

That said, although they show good performance on low-dimensional physical systems, such as single pendulums, and mass-spring oscillators, these methods face significant challenges.  
One main challenge is energy drift, where small numerical inaccuracies accumulate over time, leading to violations of conservation laws and reduced predictive accuracy.  
These issues become more pronounced as the duration of the simulation grows or as the number of degrees of freedom increases \cite{hamiltonian_neural_networks}.

A possible solution to reduce the effect of high dimensionality is model order reduction \cite{model_order_reduction}, that is, learning a lower-dimensional representation of the system’s dynamics where the IVP becomes more tractable.  
Once solved in this reduced space, the solution can then be mapped back to the original phase space.

\subsection{Model Order Reduction}

To formalize this approach, we adopt the differential geometric framework proposed in \cite{buchfink2024}, where we aim to approximate the set of all system solutions
\[
S := \{ \gamma(t; \mu) \in \mathcal{H} \mid (t, \mu) \in I \times \mathcal{P} \} \subset \mathcal{H},
\]
commonly referred to as the solution set or solution manifold, where $\mathcal{P}$ denotes the parameter space and $I$ the time interval. Rather than approximating the solution manifold $S$ over the entire high-dimensional phase space $\mathcal{H}$, our aim is to accurately learn the solution manifold within a lower-dimensional space that represents the essential dynamics. However, the existence of such a low-dimensional representation relies on the following assumption about $S$.

\begin{assumption}
   Given a metric $d_{\mathcal{H}}: \mathcal{H} \times \mathcal{H} \to \mathbb{R}_{\geq 0}$, we assume there exists a $2r$-dimensional manifold $\check{\mathcal{H}}$, with $2r \ll 2n := \dim(\mathcal{H})$, and a smooth embedding $\varphi \in C^\infty(\check{\mathcal{H}}, \mathcal{H})$ such that the image $\varphi(\check{\mathcal{H}}) \subset \mathcal{H}$ approximates the solution set $S$. Formally,
\[
d_{\mathcal{H}} \big( \varphi(\check{\mathcal{H}}), S \big) := \sup_{h \in S} \inf_{\check{h} \in \check{\mathcal{H}}} d_{\mathcal{H}} \big(h, \varphi(\check{h}) \big)
\]
is small.

Here, $\mathcal{H}$ denotes the full-order phase space, while $\check{\mathcal{H}}$ represents the reduced-order phase space. 
\end{assumption}

Intuitively, this assumption states that although the full phase space $\mathcal{H}$ may be high-dimensional, the physically realizable trajectories, the solution manifold $S$, occupy only a small region within it. This region can be well-approximated by a smooth, low-dimensional manifold embedded in $\mathcal{H}$.

Assuming we have identified such a low-dimensional manifold and its corresponding embedding \(\varphi\), the next step is to reduce the system dynamics accordingly. A naive approach might be to simply reparameterize the IVP by expressing the system state through the low-dimensional coordinates system and substituting into the full IVP. However, the original IVP consists of \(2n\) equations, while the reduced integral curve is parameterized by only \(2r \ll 2n\) variables, which means that the IVP in the reduced integral curve \(\check{\gamma}\) would be overdetermined. To eliminate redundant equations and thereby define a reduced IVP, we introduce a \emph{reduction map}.

\begin{definition}[Reduction Map \cite{buchfink2024}]
A map $R \in C^{\infty}(T\mathcal{H}, T\check{\mathcal{H}})$ is called reduction map for a smooth embedding $\varphi \in C^{\infty}(\check{\mathcal{H}}, \mathcal{H})$ if it satisfies the projection property
\begin{equation}
R \circ \mathrm{d}\varphi = \mathrm{id}_{T\check{\mathcal{H}}}.
\end{equation}
We split the reduction map
\begin{equation}
R \in C^{\infty}(T\mathcal{H}, T\check{\mathcal{H}}), \quad (h,v) \mapsto (\varrho(h), R|_h(v))
\end{equation}
with $\varrho \in C^{\infty}(\mathcal{H}, \check{\mathcal{H}})$ and $R|_h \in C^{\infty}(T_h\mathcal{H}, T_{\varrho(h)}\check{\mathcal{H}})$ for $h \in \mathcal{H}$. We refer to $\varrho$ as a point reduction for $\varphi$ and to $R|_h$ as a tangent reduction for $\varphi$.
\end{definition}
Intuitively, the reduction map lowers the dimensionality of both the points in the phase space and their associated tangent vectors. This is achieved through two operations:
\begin{itemize}
    \item \textbf{Point reduction}, carried out by the map \(\varrho\), projects points from the full-order manifold \(\mathcal{H}\) onto the reduced-order manifold \(\check{\mathcal{H}}\), and satisfies the identity
    \[
    \varrho \circ \varphi = \mathrm{id}_{\check{\mathcal{H}}},
    \]
    which we refer to as the \emph{point projection property}.

    \item \textbf{Tangent reduction}, performed by the map \(R|_m\), acts on tangent vectors at each point \(m \in \mathcal{H}\) and maps them to vectors tangent to \(\check{\mathcal{H}}\). It satisfies the identity
    \[
    R|_{\varphi(\hat{m})} \circ d\varphi|_{\hat{m}} = \mathrm{id}_{T_{\hat{m}}\check{\mathcal{H}}},
    \quad \text{for all } \hat{m} \in \check{\mathcal{H}},
    \]
    known as the \emph{tangent projection property}.
\end{itemize}

These projection properties ensure that the reduced manifold and its tangent spaces are well-defined within \(\mathcal{H}\), providing the necessary structure to formulate the reduced IVP.

\begin{definition}[Reduced-order model (ROM) \cite{buchfink2024}]
Consider the FOM, a smooth embedding $\varphi \in C^{\infty}(\mathcal{M}, \mathcal{N})$, and a reduction map $R \in C^{\infty}(T\mathcal{H}, T\check{\mathcal{H}})$ for $\varphi$ with point and tangent reduction for $\varphi$ given by $R(h,v) = (\varrho(h), R|_h(v))$. We define the ROM vector field as $\check{X}: \mathcal{P} \to \mathfrak{X}_{\mathcal{H}}$ via
\begin{equation}
\check{X}(\mu)\big|_h := R\big|_{\varphi(h)} \left( X(\mu)\big|_{\varphi(h)} \right) \in T_h\check{\mathcal{H}}.
\end{equation}

Then, for $\mu \in \mathcal{P}$, we call the IVP on $\check{\mathcal{H}}$
\begin{equation}
\begin{cases}
\frac{\mathrm{d}}{\mathrm{d}t}\check{\gamma}\big|_{t;\mu} = \check{X}(\mu)\big|_{\check{\gamma}(t;\mu)} \in T_{\check{\gamma}(t;\mu)}\check{\mathcal{H}} \\[0.5em]
\check{\gamma}(t_0; \mu) = \check{\gamma}_0(\mu) := \varrho(\gamma_0(\mu)) \in \check{\mathcal{H}}
\end{cases}
\end{equation}
the ROM for FOM under the reduction map $R$ with solution $\check{\gamma}(\cdot; \mu) \in C^{\infty}(\mathcal{I}, \check{\mathcal{H}})$.
\end{definition}

Once the reduced IVP is solved, we can use the embedding map \(\varphi\) to recover the solution of the full-order model. For a given parameter \(\mu \in \mathcal{P}\), assume that the full-order solution \(\gamma(t; \mu)\) lies on the embedded manifold \(\varphi(\check{\mathcal{H}})\) for all \(t \in I\); that is,
\[
\gamma(t; \mu) \in \varphi(\check{\mathcal{H}}), \quad \forall t \in I.
\]
Since \(\varphi \in C^\infty(\check{\mathcal{H}}, \mathcal{H})\) is a diffeomorphism onto its image, we can define a smooth reduced trajectory
\begin{equation}
\label{eq: beta}
    \check{\beta}(t; \mu) := \varphi^{-1}(\gamma(t; \mu)) \in C^\infty(I, \check{\mathcal{H}}).
\end{equation}

Applying the chain rule to the identity \(\gamma(t; \mu) = \varphi(\check{\beta}(t; \mu))\), we obtain
\[
\dot{\gamma}(t; \mu) = d\varphi|_{\check{\beta}(t; \mu)} \, \dot{\check{\beta}}(t; \mu),
\]
which implies that the full-order dynamics evaluated along the trajectory satisfy
\begin{equation}
    \label{eq: condition_2}
    X(\mu)|_{\gamma(t; \mu)} = d\varphi|_{\check{\beta}(t; \mu)} \, \dot{\check{\beta}}(t; \mu).
\end{equation}

Thus, the reduced vector field must reproduce the trajectory \(\check{\beta}(t; \mu)\). Now, we can show that the reduced trajectory \(\check{\beta}(t; \mu)\) is recovered by the ROM.

\begin{theorem}[Exact reproduction of a solution \cite{buchfink2024}]
Assume that the FOM  is uniquely solvable and consider a reduction map $R \in C^{\infty}(T\mathcal{H}, T\check{\mathcal{H}})$ for the smooth embedding $\varphi \in C^{\infty}(\check{\mathcal{H}}, \mathcal{H})$ and a parameter $\mu \in \mathcal{P}$. Assume that the ROM is uniquely solvable and $\gamma(t; \mu) \in \varphi(\check{\mathcal{H}})$ for all $t \in \mathcal{I}$. Then the ROM solution $\check{\gamma}(\cdot; \mu)$ exactly recovers the solution $\gamma(\cdot; \mu)$ of the FOM  for this parameter, i.e.,
\begin{equation}
\varphi(\check{\gamma}(t; \mu)) = \gamma(t; \mu) \quad \text{for all } t \in \mathcal{I}.
\end{equation}
\end{theorem}

This confirms that the reduced dynamics can reproduce the full-order solution on \(\varphi(\check{\mathcal{H}})\) when lifted through the embedding \(\varphi\).
The choice of embedding \(\varphi\) is therefore critical to the approximation quality of the ROM. One common approach is to assume that the reduced manifold \(\check{\mathcal{H}}\) is a low-dimensional linear subspace of the full phase space \(\mathcal{H}\). Under this linear assumption, the embedding \(\varphi\) corresponds to a linear projection onto the subspace, as employed by classical model reduction techniques such as Proper Orthogonal Decomposition (POD) \cite{berkooz_pod}.
\begin{example}
\label{ex: pod}
\textit{Projection-based linear-subspace MOR with a reduced-basis matrix $V \in \mathbb{R}^{2n \times 2r}$ and a projection matrix $W \in \mathbb{R}^{2n \times 2r}$ is contained as a special case of the presented formulation with $\mathcal{H} = \mathbb{R}^{2n}$, $\check{\mathcal{H}} =  \mathbb{R}^{2r}$, $x = \mathrm{id}_{\mathbb{R}^{2n}}$, $\check{x} = \mathrm{id}_{\mathbb{R}^{2r}}$ and}
\begin{align*}
\varrho(h) := W^{\top} h, \quad R|_h(v) := W^{\top} h, \quad \varphi(\check{h}) := V \check{h}.
\end{align*}
\textit{This exactly covers the case where $\varphi$ and $\varrho$ are linear. The projection property then relates to the biorthogonality of $W$ and $V$}
\begin{align*}
  \varrho \circ \varphi \equiv \mathrm{id}_{\mathbb{R}^{2r}} \quad &\Longleftrightarrow \quad W^{\top} V = I_{2r} \in \mathbb{R}^{2r \times 2r},  \\
R|_{\varphi(\hat{h})} \circ d\varphi|_{\hat{h}} \equiv \mathrm{id}_{\mathbb{R}^{2r}} \quad &\Longleftrightarrow \quad W^{\top} V = I_{2r} \in \mathbb{R}^{2r \times 2r},
\end{align*}
\textit{which is the assumption underlying POD.}
\end{example}
Although POD is effective for problems where the solution manifold is approximately linear, it may fail to capture the complexity of inherently nonlinear solution sets.
To address this limitation, nonlinear embedding methods use neural networks to learn complex reduced representations. Among these, autoencoders are widely used: the encoder projects states to a latent space \(\check{\mathcal{H}}\), and the decoder reconstructs the original states, effectively learning a nonlinear embedding.

Once an embedding \(\varphi\) is chosen, whether linear or nonlinear, the next step is to construct the associated reduction map.

\begin{theorem}[Manifold Galerkin Projection \cite{buchfink2024}]
\label{them: manifold_galerkin_projection}
\textit{Consider a smooth embedding $\varphi$ and a point reduction $\varrho$ for $\varphi$. Then, the differential of the point reduction $\varrho$ is a left inverse to the differential of the embedding $\varphi$. Consequently,}
\begin{equation}
R_{\text{MPG}}: T\mathcal{H} \to T\check{\mathcal{H}} \quad (h, v) \mapsto (\varrho(h), d\varrho|_h(v)) 
\end{equation}
\textit{is a smooth reduction map for $\varphi$, which we call the MPG reduction map for $(\varrho, \varphi)$.}
\end{theorem}

One major limitation of this model reduction technique is that it often neglects the geometry of the original manifold. For instance, if the original manifold is a phase space, which is naturally endowed with a symplectic structure, the reduction must preserve this symplecticity, regardless of the chosen embedding. However, if a linear embedding is used, the corresponding reduction map is an orthogonal projection matrix (Example \ref{ex: pod}), as in POD.
\begin{example}
\textit{If $\varphi$ and $\varrho$ are linear, then the MPG-ROM with the MPG reduction map from Theorem \ref{them: manifold_galerkin_projection} is the ROM obtained in classical linear-subspace MOR via Petrov--Galerkin projection}
\begin{equation}
R_{\text{MPG}}|_{\varphi(m)} = D\varrho|_{\varphi(m)} = W^{\top}, \quad \frac{d}{dt}\tilde{\gamma}\bigg|_t = W^{\top} X|_{\tilde{\gamma}(t)},
\end{equation}
\textit{which is the motivation for the terminology MPG.} 
\end{example}
Although effective in some cases, POD does not guarantee the preservation of the symplectic geometry of the full-order phase space. Thus, the reduced-order phase space is not necessarily symplectic. Therefore, it is necessary to define reduction maps that preserve geometric structures, such as symplectic geometry. 

\begin{theorem}[Generalized Manifold Galerkin (GMG) \cite{buchfink2024}]
\label{thm: generalized_manifold_galerkin}
Let $\mathcal{M}$ be a manifold of dimension $N$ endowed with a non-degenerate $(0,2)$-tensor field $\tau \in \Gamma(T^{(0,2)}(T\mathcal{M}))$, and let $\varphi \in \mathcal{C}^{\infty}(\tilde{\mathcal{M}}, \mathcal{M})$ be a smooth embedding such that the reduced tensor field $\tilde{\tau} := \varphi^*\tau \in \Gamma(T^{(0,2)}(T\tilde{\mathcal{M}}))$ is nondegenerate. 

Define the vector bundle $E_{\varphi(\tilde{\mathcal{M}})} := \bigcup_{m \in \varphi(\tilde{\mathcal{M}})} T_m\mathcal{M}$ and the generalized manifold Galerkin (GMG) mapping as follows:
\begin{equation*}
   R_{\text{GMG}}: T\mathcal{M} \supseteq E_{\varphi(\tilde{\mathcal{M}})} \to T\tilde{\mathcal{M}}, \quad (m,v) \mapsto \left(\varrho(m), \left(\sharp_{\check{\tau}} \circ \mathrm{d}\varphi^*|_{\varrho(m)} \circ b_{\tau}\right)(v)\right), 
\end{equation*}
where $\varrho: \varphi(\tilde{\mathcal{M}}) \to \tilde{\mathcal{M}}$ is the inverse of $\varphi$, $b_{\tau}$ is the musical isomorphism \footnote{A \emph{musical isomorphism} is the canonical isomorphism between vectors and covectors induced by a nondegenerate \((0,2)\)-tensor \(T\) on a manifold \(\mathcal{M}\). The \emph{flat} map sends a vector \(v\) to the covector
\[
v^\flat = T(v, \cdot),
\]
and the \emph{sharp} map is its inverse, assigning to each covector \(\alpha\) the unique vector \(\alpha^\sharp\) satisfying
\[
T(\alpha^\sharp, \cdot) = \alpha.
\]
} induced by $\tau$, and $\sharp_{\tau}$ is the musical isomorphism induced by $\tilde{\tau}$. Then $R_{\text{GMG}}$ is a reduction map for $\varphi$.
\end{theorem}
The geometric structure preserved by GMG  is characterized by a reduction of the manifold tensor field.
A tensor field is a geometric object that maps each point on a manifold to a multilinear map (a tensor) \cite{grinfeld2013introduction}. The geometry of a manifold is largely determined by these tensor fields. For instance, a Riemannian metric is a symmetric, positive-definite (0,2)-tensor field that defines distances and angles on the manifold. Similarly, a symplectic form is a nondegenerate, closed 2-form (a skew-symmetric (0,2)-tensor field) that encodes the geometry of phase spaces in Hamiltonian mechanics. In particular, for Hamiltonian systems, it is crucial that the symplectic form is correctly reduced, a property that our approach explicitly enforces in reduction map construction.

\section{Reduced-Order Geometric Hamiltonian Neural Networks}
Prior work on learning Hamiltonian system dynamics often overlooks their underlying geometry. In particular, two geometric structures must be accounted for during the learning process: the \emph{Riemannian geometry} of the inertia matrices, and the \emph{symplectic geometry} of the phase space.

\subsection{Riemannian Geometry of Inertia Matrix}
We begin by considering the inertia matrix. Existing approaches to learning the Hamiltonian function generally follow one of two strategies. Some methods model the Hamiltonian as a whole, without explicitly separating kinetic and potential energy components \cite{hamiltonian_neural_networks}. Other approaches explicitly decompose the Hamiltonian into kinetic and potential energies and learn them separately \cite{zhong2024symplecticodenetlearninghamiltonian}. For the kinetic energy, this requires learning the inertia matrix. Since it must be symmetric positive definite (SPD), they typically satisfy this constraint using one of the following methods:

\begin{enumerate}
    \item \textbf{Symmetrization:} The predicted matrix is symmetrized as
    \[
        M = \tfrac{1}{2}(M + M^\top) + \varepsilon I,
    \]
    where $\varepsilon > 0$ ensures positive definiteness. Although this guarantees an SPD matrix, it treats $M$ as an unconstrained Euclidean object, ignoring the Riemannian geometry of the SPD manifold. Thus, the predicted matrices may not evolve naturally along the SPD manifold, leading to suboptimal parameter updates that follow arbitrary Euclidean directions rather than geodesics.
    
    \item \textbf{Cholesky factorization:} The inertia matrix is expressed as
    \[
        M = L L^\top,
    \]
    where $L$ is a lower triangular matrix with non-zero diagonal elements. This ensures $M$ is SPD, but it implicitly assigns a Euclidean geometry to the SPD space via the space of Cholesky factors. This oversimplification neglects the  Riemannian geometry of the SPD manifold $\mathcal{S}_{++}^n$, potentially leading to the \emph{swelling effect}~\citep{feragen2017geodesic,lin2019riemannian}, where interpolating between SPD matrices with the same determinant might produce intermediate matrices with increased determinant. This geometric distortion can degrade the physical plausibility of the learned dynamics.
\end{enumerate}

To address these limitations, we propose a \emph{geometry-aware} parameterization of the inertia matrix $M(q) \in \mathcal{S}_{++}^n$ by modeling it directly as a point on the SPD manifold. This requires accounting for the non-Euclidean geometry of $\mathcal{S}_{++}^n$, which cannot be captured by Euclidean distances. To this end, we adopt the \emph{Affine Invariant Metric} (AIM)~\citep{spd_manifold_metric}, which defines a natural Riemannian distance between two SPD matrices $M_1, M_2 \in \mathcal{S}_{++}^n$ as:
\[
    d_{\mathrm{AIM}}(M_1, M_2) = \left\| \log\left(M_1^{-1/2} M_2 M_1^{-1/2}\right) \right\|_F,
\]
where $\log(\cdot)$ denotes the matrix logarithm and $\|\cdot\|_F$ is the Frobenius norm. This metric is invariant under congruence transformations (change of basis) $M \mapsto P^\top M P$ for any invertible matrix $P$. Such invariance ensures coordinate-independence, an essential property in physical systems where the dynamics must remain unchanged under changes of coordinates, making the AIM a suitable metric.

Having specified the metric on our manifold, we now define the exponential map at a base point $M_0 \in \mathcal{S}_{++}^n$. The exponential map takes a tangent vector $\Xi \in T_{M_0} \mathcal{S}_{++}^n$ and maps it to a point on the SPD manifold by moving along the geodesic starting at $M_0$ in the direction of $\Xi$, ensuring the result remains SPD. Formally, it is given by:
\[
    \mathrm{Exp}_{M_0}(\Xi) = M_0^{1/2} \exp\left( M_0^{-1/2} \Xi M_0^{-1/2} \right) M_0^{1/2},
\]
where $\exp(\cdot)$ denotes the matrix exponential. This allows the parameterization of $M$ as a smooth perturbation from a base matrix $M_0$ along the SPD manifold. In our work, $M_0$ is treated as a learnable base matrix, and the neural network predicts the tangent vector $\Xi$, ensuring that $M = \mathrm{Exp}_{M_0}(\Xi)$ always lies on the SPD manifold and respects its geometry. Building on this geometry-aware parameterization, we introduce our \emph{Geometric Hamiltonian Neural Network} (GeoHNN), which consists of two main components:
\begin{itemize}
    \item \textbf{Inertia Network} $\mathcal{N}_M$: Given a configuration $q \in \mathcal{Q}$, this network outputs a tangent vector $\Xi(q) \in T_{M_0} \mathcal{S}_{++}^n$, from which the inverse inertia matrix is reconstructed as:
    \[
        M(q)^{-1} = \mathrm{Exp}_{M_0} \left( \mathcal{N}_M(q; \theta_M) \right),
    \]
    where $\theta_M$ and $M_0$ are the network parameters.

    \item \textbf{Potential Energy Network} $\mathcal{N}_V$: A standard feedforward neural network that predicts the potential energy:
    \[
        V(q) = \mathcal{N}_V(q; \theta_V).
    \]
    where $\theta_V$ are the network parameters.

\end{itemize}

The total Hamiltonian is then defined as:
\[
    H(q, p) = \frac{1}{2} p^\top M(q)^{-1} p + V(q).
\]

Although our GeoHNN captures the geometry of the SPD manifold, its computational complexity scales poorly in high-dimensional systems where $n \gg 1$. In such regimes, the cost of computing and storing full-rank inertia matrices, as well as performing manifold-aware operations (e.g., matrix exponentials), becomes prohibitive. To address this issue, we need a model reduction method that reduces computation while keeping the symplectic structure of the phase space.

\subsection{Symplectic geometry of the Phase Space}
To preserve the geometry of the phase space, we apply Theorem \ref{thm: generalized_manifold_galerkin} to Hamiltonian systems in order to preserve the symplectic geometry of the phase space. The phase space is naturally endowed with a canonical symplectic form $\omega = dp \wedge dq$. This symplectic form is the tensor field governing the phase space geometry, and it needs to be properly reduced to ensure the symplecticity of the reduced-order model. Under a smooth embedding $\varphi \in C^{\infty}(\check{\mathcal{H}}, \mathcal{H})$, the reduced phase space $\check{\mathcal{H}}$ can be endowed with a reduced symplectic form $\hat{\omega} = \varphi^* \omega$.
\begin{lemma}[\cite{buchfink2024}]
\textit{Consider a symplectic manifold $(\mathcal{H}, \omega)$, a smooth manifold $\check{\mathcal{H}}$, and a smooth embedding $\varphi \in C^{\infty}(\check{\mathcal{H}}, \mathcal{H})$ such that $\check{\omega} := \varphi^* \omega$ is nondegenerate. Then $\check{\omega}$ is a symplectic form, $(\check{\mathcal{H}}, \check{\omega})$ is a symplectic manifold, and $\varphi$ is a symplectomorphism.}
\end{lemma}

Assuming this symplectic form is non degenerate, $\varphi$ is a symplectomorphism, that is, a diffeomorphism that preserves the symplectic form, and we can use it to define our reduction map.

\begin{equation*}
\begin{array}{rcl}
& T\mathcal{H} \supset E_{\varphi}(\mathcal{H}) & \longrightarrow T\tilde{\mathcal{H}} \\
R_{\text{SMG}}: & (h, v) & \longmapsto \left( \varrho(h), \left({\musSharp_{\check{\omega}}}\circ {\mathrm{d}\varphi^*\big|_{\varrho(h)}} \circ {\musFlat_{\omega}} \right)(v) \right)
\end{array}
\end{equation*}
We refer to this reduction map as the Symplectic Manifold Galerkin Projection associated with the embedding $\varphi$, and it defines our ROM. Now, the final step is to verify that the resulting system indeed constitutes a valid Hamiltonian system in the reduced phase space.
\begin{theorem}[Symplectic Manifold Galerkin \cite{buchfink2024}]
\label{thm: SMG-ROM}
\textit{The SMG-ROM is a Hamiltonian system $(\check{\mathcal{H}}, \tilde{\omega}, \check{H})$ with the reduced Hamiltonian $\check{H} := \varphi^* H = H\circ \varphi$.}    
\end{theorem}
Theorem \ref{thm: SMG-ROM} ensures that, under this reduction, the resulting system preserves the symplectic structure and thus defines a valid Hamiltonian system on the reduced phase space for the chosen embedding $\varphi$.
To realize the embedding $\varphi : \check{\mathcal{H}} \mapsto \mathcal{H}$ in practice, we parameterize it using an autoencoder architecture. Concretely, we define:
\begin{itemize}
    \item \textbf{Encoder} $\rho_{\text{AE}} : \mathbb{R}^{2n} \rightarrow \mathbb{R}^{2r}$, which maps the high-dimensional phase space to a reduced latent space;
    \item \textbf{Decoder} $\varphi_{\text{AE}} : \mathbb{R}^{2r} \rightarrow \mathbb{R}^{2n}$, which reconstructs the full phase space from the latent space.
\end{itemize}
The composition $\varphi_{\text{AE}} \circ \rho_{\text{AE}} : \mathbb{R}^{2n} \rightarrow \mathbb{R}^{2n}$ forms an autoencoder, and the intermediate representation in $\mathbb{R}^{2r}$ serves as the reduced-order phase space. We denote the family of these autoencoders as
\begin{equation*}
   \mathcal{F}_{\varphi, \rho, \text{AE}} := \left\{ (\varphi_{\text{AE}}, \rho_{\text{AE}}) \mid \theta \in \mathbb{R}^{n_\theta} \text{ are network parameters} \right\}. 
\end{equation*}
Although autoencoders are highly flexible, they do not, in general, guarantee the pointwise projection property or the exact preservation of symplectic structure. To guarantee \emph{weak symplecticity}, meaning that the symplectic structure is preserved up to a small controlled error, prior work such as \cite{buchfink2021} incorporates geometric regularization terms during training to guarantee \emph{weak} symplecticity preservation.

To strictly preserve this symplecticity, we constrain the layers of the autoencoder architecture to strictly satisfy the point projection property \cite{otto2023}. In our framework, the encoder and decoder are built from layers that work in pairs, where each layer in the encoder has a corresponding inverse layer in the decoder. Each encoder layer is defined as:
\[
\rho_{\mathcal{H}}^{(l)}(h^{(l)}) = \sigma^{-}\left( \Psi_l^\top (h^{(l)} - b_l) \right),
\]
while the corresponding decoder layer is given by
\[
\varphi_{\mathcal{H}}^{(l)}(\hat{h}^{(l-1)}) = \Phi_l \, \sigma^{+}(\hat{h}^{(l-1)}) + b_l,
\]
where $(\Phi_l, \Psi_l)$ are weight matrices, $b_l$ are bias vectors, and $(\sigma^+, \sigma^-)$ are smooth, invertible activation functions.

These weight matrices and activation functions are required to satisfy the conditions
\[
\Psi_l^\top \Phi_l = I \quad \text{and} \quad \sigma^- \circ \sigma^+ = \mathrm{id},
\]
to ensure that $\rho_{\mathcal{H}}^{(l)} \circ \varphi_{\mathcal{H}}^{(l)} = \mathrm{id}_{\mathbb{R}^n}$, and consequently, the entire autoencoder satisfies the map reduction property.

To maintain the biorthogonality constraint during training, we take a geometric approach, rather than relying on indirect methods such as overparameterization and penalty terms in the loss function, as used by \cite{otto2023}. Specifically, we treat each pair $(\Phi_l, \Psi_l)$ as a point on the biorthogonal manifold defined as follows:
\begin{equation*}
    \mathcal{B}_{n_l, n_{l-1}} := \left\{ (\Phi, \Psi) \in \mathbb{R}^{n_l \times n_{l-1}} \times \mathbb{R}^{n_l \times n_{l-1}} \;\middle|\; \Psi^\top \Phi = I_{n_{l-1}} \right\}
\end{equation*}
and we perform optimization over this manifold to minimize the reconstruction loss (i.e., the difference between the input and its reconstruction by the autoencoder), while strictly enforcing the biorthogonality constraint. In addition, the invertibility of the nonlinear activations constraint (i.e., $\sigma^- \circ \sigma^+ = \mathrm{id}$), it is satisfied by using the smooth, bijective activation functions proposed in \cite[Equation 12]{otto2023}.

Finally, to effectively leverage our geometric priors, namely, the inertia matrix parameterization and the constrained autoencoder, we formulate a training objective that aligns the model with the underlying geometry of Hamiltonian systems.

\subsection{Training Objective and Optimization Procedure}
For low-dimensional systems without dimensionality reduction, training involves minimizing the mean squared error between the predicted temporal derivatives of the states \((q, p)\) and their ground truth values. In this setting, the total loss reduces to
\[
\mathcal{L}_{\text{total}} = \sum_{i,j} \left\| \dot{\tilde{q}}_i(t_j) - \dot{q}_i(t_j) \right\|^2 + \left\| \dot{\tilde{p}}_i(t_j) - \dot{p}_i(t_j) \right\|^2,
\]
where \(\dot{q}_i, \dot{p}_i\) denote the true temporal derivatives and \(\dot{\tilde{q}}_i, \dot{\tilde{p}}_i\) the predicted ones. 

For high-dimensional systems with dimensionality reduction, we design a composite loss function that captures the essential goals of reduced-order modeling: long-term predictive accuracy (trajectory $q(t)$), geometric consistency of the latent space (biorthogonality), and fidelity of reconstruction of the states ($q(t), p(t)$) . Each component contributes to preserving the Hamiltonian structure in the reduced phase space.

\paragraph{Multi-step Prediction Loss.}
Since the dynamics evolve over time, relying solely on single-step prediction losses is insufficient to ensure stability and accuracy. Instead, we adopt a multi-step prediction loss, where we simulate the system forward in the reduced latent phase space using a symplectic Euler integrator over $N$ steps with time step $\Delta t$, generating a predicted trajectory $\{(\tilde{q}_i(t_j), \tilde{p}_i(t_j))\}_{j=1}^N$. The predicted trajectory is then compared against ground truth trajectories $\{(q_i(t_j), p_i(t_j))\}_{j=1}^N$. This yields the loss term:
\[
\mathcal{L}_{\text{multistep}} = \sum_{i,j} \left\| \tilde{q}_i(t_j) - q_i(t_j) \right\|^2 + \left\| \tilde{p}_i(t_j) - p_i(t_j) \right\|^2.
\]

\paragraph{Latent Encoding Loss.}
Accurate rollout does not guarantee that the latent variables $(\check{q}, \check{p})$ produced by the model correspond to physically meaningful coordinates. To ensure that the learned encodings preserve physical structure, we supervise them using ground truth latent values via:
\[
\mathcal{L}_{\text{latent}} = \sum_{i,j} \left\| \check{q}_i(t_j) - \rho_{\mathcal{Q}}(q_i(t_j)) \right\|^2 + \left\| \check{p}_i(t_j) - \rho_{\mathcal{P}}(p_i(t_j)) \right\|^2.
\]

\paragraph{Reconstruction Loss.}
To preserve the geometry of the phase space through the point projection property, i.e., $\varphi_{\mathcal{Q}} \circ \rho_{\mathcal{Q}} = \mathrm{id}$ and $\varphi_{\mathcal{P}} \circ \rho_{\mathcal{P}} = \mathrm{id}$, we require encoders and decoders to be inverse mappings. This is enforced by minimizing the reconstruction error:
\[
\mathcal{L}_{\text{recon}} = \sum_{i,j} \left\| \varphi_{\mathcal{Q}}(\check{q}_i(t_j)) - q_i(t_j) \right\|^2 + \left\| \varphi_{\mathcal{P}}(\check{p}_i(t_j)) - p_i(t_j) \right\|^2.
\]

\paragraph{Regularization.}
To prevent overfitting and encourage generalization, we apply standard $\ell_2$ regularization on model parameters $\theta$:
\[
\mathcal{L}_{\text{reg}} = \gamma \| \theta \|^2,
\]
where $\gamma > 0$ is a regularization coefficient.

\paragraph{Full Loss Function.}
The total loss function minimized during training is:
\[
\mathcal{L}_{\text{total}} =  \mathcal{L}_{\text{multistep}} +  \mathcal{L}_{\text{latent}} + \mathcal{L}_{\text{recon}} + \mathcal{L}_{\text{reg}},
\]

\paragraph{Riemannian Optimization.}
To optimize $\mathcal{L}_{\text{total}}$ while respecting the geometry of the underlying parameter manifolds, namely, the SPD manifold for inertia matrices and the biorthogonal manifold for encoder–decoder weight pairs, we employ Riemannian optimization techniques. These algorithms operate on curved manifolds by updating parameters along geodesic paths.

Each Riemannian optimization step consists of the following operations:

\begin{itemize}
    \item \textbf{Riemannian Gradient Computation:} Compute the Euclidean gradient $\nabla_{\theta} \mathcal{L}$ and project it onto the tangent space $T_M \mathcal{M}$ of the manifold $\mathcal{M}$ at the current point $M$ to obtain a valid tangent direction respecting the geometry of the manifold $\mathcal{M}$.
    
    \item \textbf{Retraction:} Move in the tangent direction along a geodesic and map the result back to the manifold via a retraction operator. When the latter is the exponential map, the update strictly follows the shortest path on the manifold.
    
    \item \textbf{Vector Transport (for momentum-based methods):} For optimizers like Riemannian Adam or Riemannian SGD with momentum, past velocity vectors must be transported from the old tangent space to the new one using a vector transport operation.
\end{itemize}

These operations ensure that model parameters remain on their respective manifolds throughout training, preserving the required geometric structure. Detailed derivations and algorithmic implementations are provided in the Appendix.

\section{Experiments}
\subsection{Experimental Setup}  
All models were implemented in PyTorch. For optimization, we used the AdamW optimizer \cite{loshchilov2019decoupledweightdecayregularization} with a weight decay of \(5 \times 10^{-4}\) for standard architectures, and Riemannian Adam optimizer \cite{bécigneul2019riemannianadaptiveoptimizationmethods} (via the \texttt{geoopt} library \cite{geoopt2020kochurov}) for GeoHNN model to accommodate manifold-constrained parameters. Training was performed for a maximum of 1000 epochs with early stopping based on validation loss, using a patience of 50 epochs and a minimum relative improvement threshold of \(10^{-6}\). Learning rates were selected from the set \(\{10^{-3}, 5 \times 10^{-4}, 10^{-4}, 5 \times 10^{-5}, 10^{-5}\}\) through manual tuning per model and dataset to ensure stable convergence.

Datasets were split into 80\% training, 10\% validation, and 10\% test sets. A batch size of 128 was used across all experiments. Final model selection was based on the lowest validation loss, and all reported performance metrics are computed on the held-out test set. Training was conducted on an NVIDIA RTX8000 GPU. To assess robustness and statistical significance, each model was trained over five independent runs, and we report the mean and standard deviation of the resulting test metrics: energy drift, and trajectory error.

\subsection{Models Description}
To assess the impact of geometric priors on learning Hamiltonian dynamics, we evaluate a diverse set of machine learning models. These models vary in how they incorporate such priors and are categorized based on their suitability for low- and high-dimensional systems.

\subsubsection{Low-Dimensional Setting}

\begin{itemize}
    \item \textbf{Baseline MLP} A fully-connected multi-layer perceptron that directly models the time derivative \(\dot{x} = f(x)\) without any physical or geometric inductive bias. This model serves as a reference to quantify the benefits of incorporating Hamiltonian structure and geometry in learning.
    \item \textbf{Vanilla HNN)} The original HNN formulation \citep{hamiltonian_neural_networks}, which parameterizes the Hamiltonian function \(H(q, p)\) via a neural network and derives the system dynamics by applying Hamilton's equations. This model incorporates the physical Hamiltonian prior but does not explicitly enforce geometric constraints.
    \item \textbf{DoubleHead HNN} An extension of the standard HNN that explicitly models kinetic and potential energies via two separate neural network branches. One branch learns the kinetic energy term \(T(q, p)\), through a learned symmetrized mass matrix, while the other models the potential energy \(V(q)\).
    \item \textbf{Cholesky-Parameterized HNN} Another variant that ensures positive definiteness of the mass matrix \(M(q)\) by parameterizing it via its Cholesky decomposition, \(M(q) = L(q) L(q)^\top\). This guarantees physically valid inertia matrices, though it does not explicitly respect the SPD manifold geometry.
    \item \textbf{GeoHNN} Our proposed model that explicitly parameterizes \(M(q)\) as a point on the symmetric positive definite (SPD) manifold. By leveraging Riemannian optimization, GeoHNN preserves the manifold geometry, ensuring physically consistent and valid mass matrices.
\end{itemize}

\subsubsection{High Dimensional Setting}
In high-dimensional systems, we use GeoHNN to learn Hamiltonian dynamics by modeling the inertia matrix \(M(q)\) on the SPD manifold with Riemannian optimization, and perform dimensionality reduction using different choices of autoencoders.

\paragraph{Vanilla Autoencoder}
This autoencoder takes the concatenated position and momentum vectors \((q, p)\) as input and encodes them into a latent vector of size \(2r\). GeoHNN is applied in the latent space to model the reduced Hamiltonian dynamics efficiently.

\paragraph{Biorthogonal Autoencoder}
A geometry-preserving autoencoder that enforces biorthogonality constraints between encoder and decoder. This ensures the latent space satisfies the reduction map property.

\subsection{Dataset Description}

We evaluate our models on a diverse set of low- to high-dimensional physical systems. These systems vary in size and complexity, allowing us to assess the models’ ability to learn geometry-preserving dynamics that respect conservation laws while scaling effectively. 

\subsubsection{Low-Dimensional Systems}
In low-dimensional settings, we consider the following physical systems:
\begin{itemize}
    \item \textbf{Mass-Spring System} A point mass attached to a linear spring oscillating in one dimension. This system features linear dynamics, and a constant inertia matrix.
    \item \textbf{Coupled Oscillators} Two masses connected by springs in a one-dimensional arrangement. The system exhibits harmonic coupling between degrees of freedom. It is used to evaluate the model's ability to capture interactions between subsystems in a conservative setting.
    \item \textbf{Two-Body Problem} Two particles interacting via Newtonian gravity in a two-dimensional plane. The dynamics follow an inverse-square central force law and conserve both energy and angular momentum. 
    \item \textbf{Single Pendulum} A single rigid body rotating in a vertical plane under the influence of gravity. This system has one degree of freedom and exhibits periodic motion with conserved energy. It serves as a classical example of nonlinear Hamiltonian dynamics.
\end{itemize}

\subsubsection{High-Dimensional Systems}
In high-dimensional settings, we consider a \textbf{\emph{deformable cloth}} \cite{friedl2025riemannianframeworklearningreducedorder}, which is a two-dimensional mass-spring model. Each vertex is treated as a point mass connected to its neighbors by springs, resulting in a high number of degrees of freedom. This system challenges the model’s ability to learn geometry-preserving dynamics efficiently in high-dimensional settings.

\subsection{Numerical Results}
Our experiments across various physical systems consistently demonstrate that incorporating geometric priors significantly improves learning accuracy, long-term stability, and energy conservation in Hamiltonian systems.
\subsubsection{Low Dimensional Setting}
Firstly, we begin by evaluating our method in low-dimensional settings to isolate and understand the impact of incorporating geometric priors on the inertia matrix.
\paragraph{Mass Spring}
Figure~\ref{fig:mass_spring_trajectory_error} shows that the MLP baseline exhibits rapid error growth on the mass-spring system, with trajectory error increasing from \(10^{-2}\) to nearly \(10^{0}\) over a 50-unit integration window. In contrast, all HNN variants (vanilla HNN, DoubleHeadHNN, CholeskyHNN, and GeoHNN) maintain bounded errors within \(10^{-3}\) to \(10^{-2}\). 
\begin{figure}[H]
    \centering
    \includegraphics[width=0.65\columnwidth]{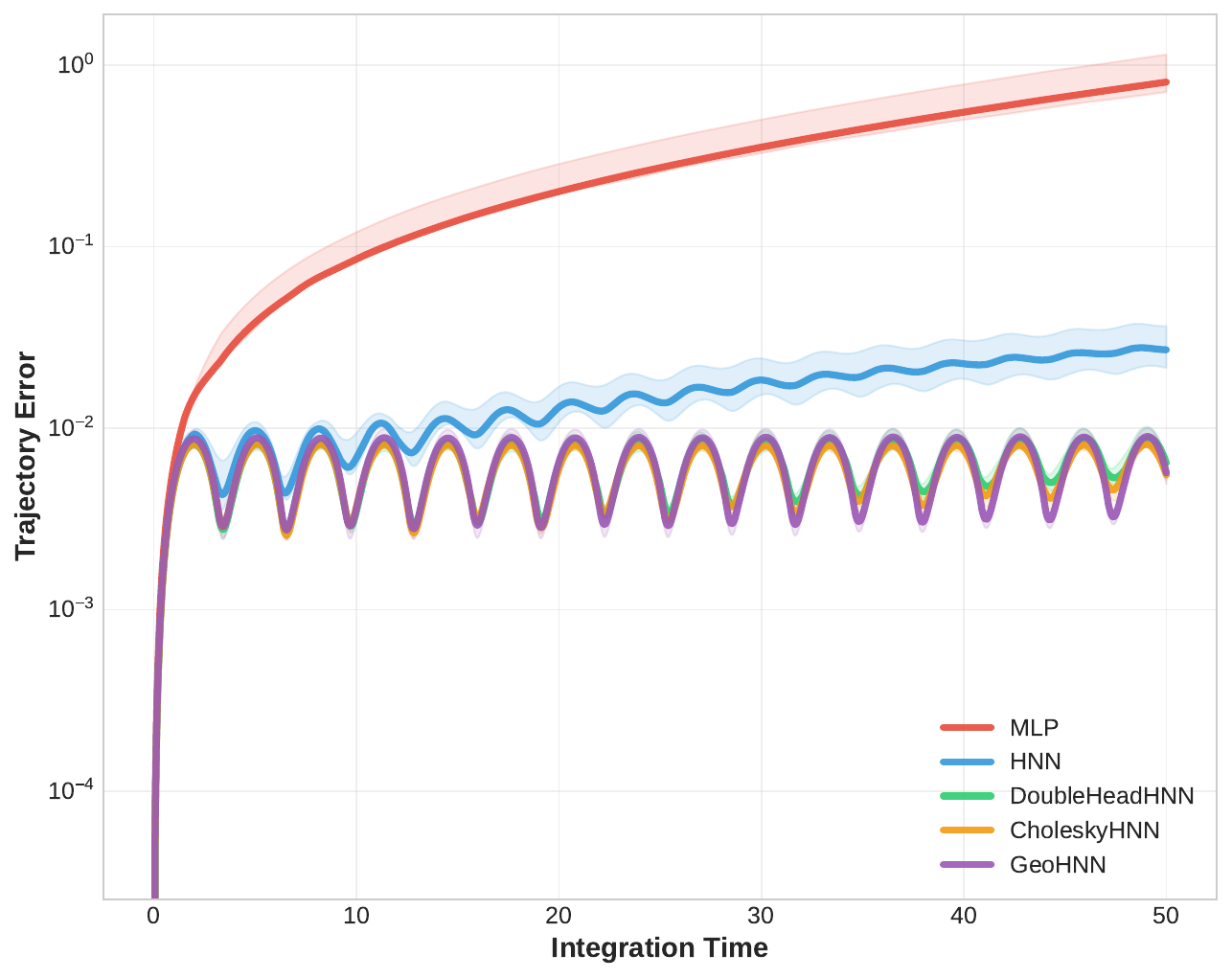}
    \caption{Long-term trajectory error as a function of time for the mass-spring system.}
    \label{fig:mass_spring_trajectory_error}
\end{figure}
Figure~\ref{fig:mass_spring_energy} further illustrates that the MLP suffers from unbounded relative energy drift, reaching values near \(10^{1}\), whereas the HNN variants exhibit oscillatory, bounded energy errors around \(10^{-1}\). Among these, GeoHNN achieves the lowest and most stable errors across both metrics, followed by CholeskyHNN.
\begin{figure}[H]
    \centering
    \includegraphics[width=0.65\columnwidth]{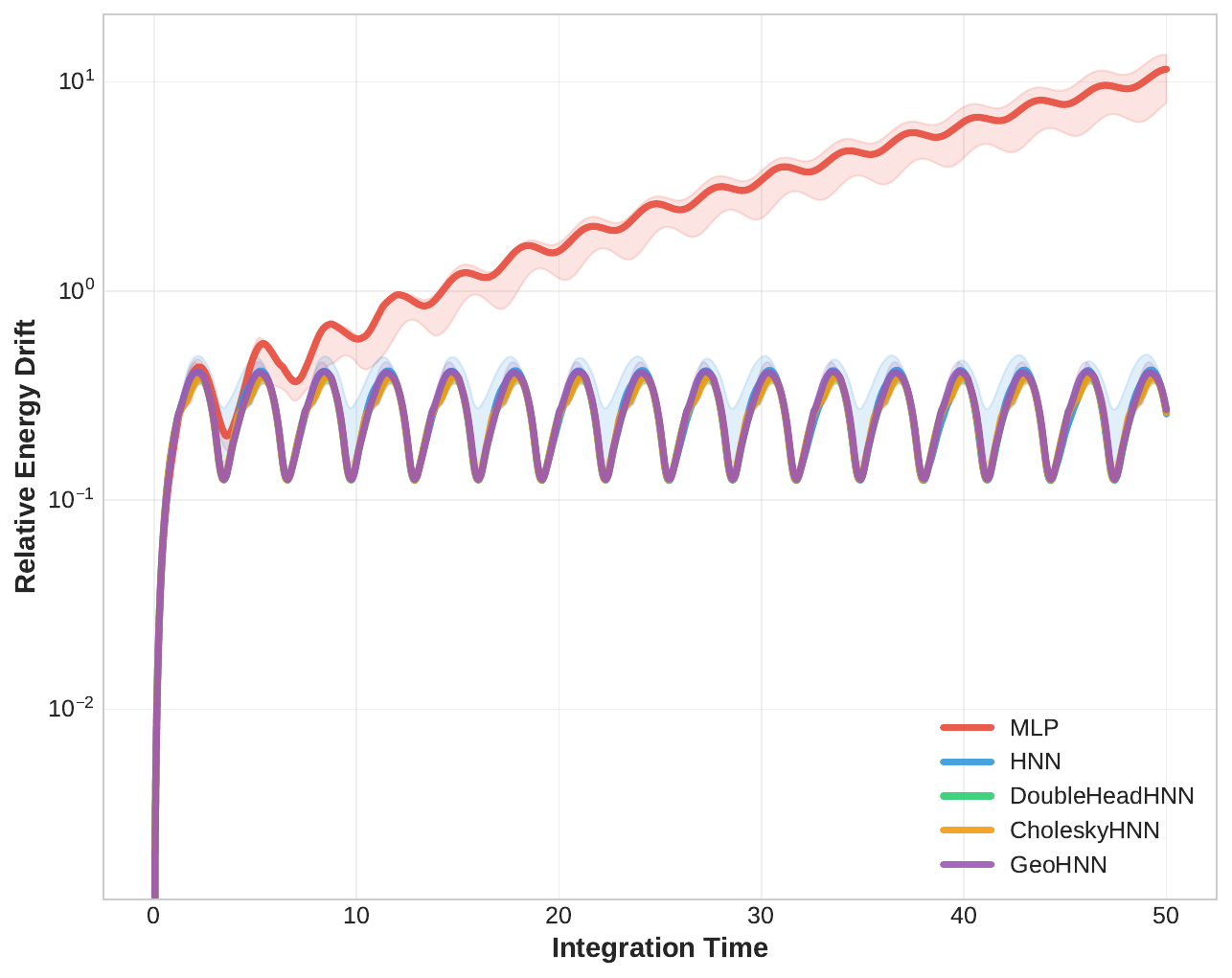}
    \caption{Energy drift as a function of time for the mass-spring system.}
    \label{fig:mass_spring_energy}
\end{figure}

\paragraph{3-Coupled Oscillators}
As shown in Figure~\ref{fig:coupled_oscillators_trajectory_error}, the MLP baseline exhibits exponential error growth from \(10^{-2}\) to \(10^{0}\), leading to unstable long-term predictions. In contrast, structure-preserving models remain stable, with HNN reaching \(\sim 5 \times 10^{-2}\), DoubleHeadHNN and CholeskyHNN improving to \(\sim 2 \times 10^{-2}\), and GeoHNN achieving the lowest errors near \(10^{-2}\). 
\begin{figure}[H]
    \centering
    \includegraphics[width=0.65\columnwidth]{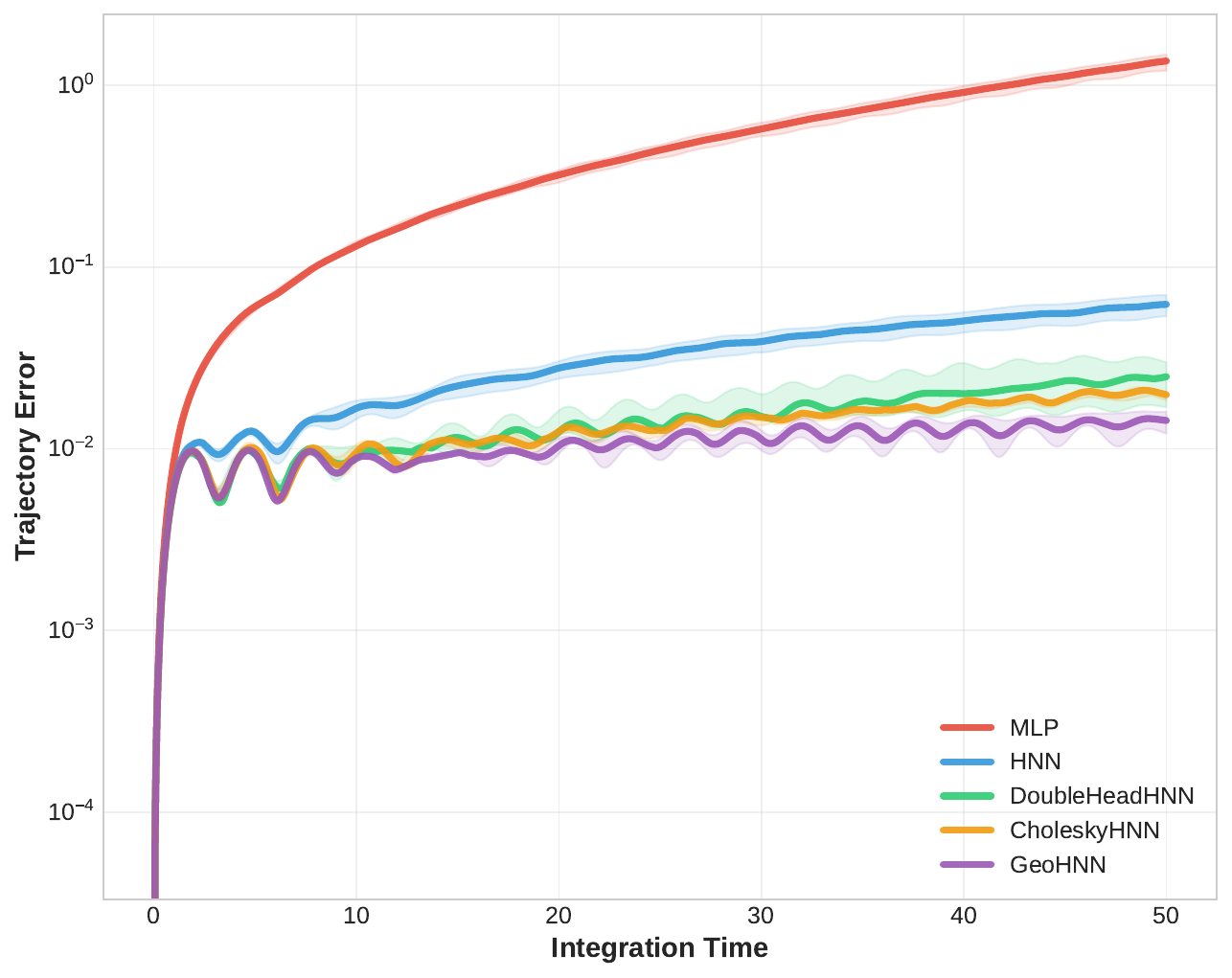}
    \caption{Long-term trajectory error as a function of time for $3$-coupled oscillators system.}
    \label{fig:coupled_oscillators_trajectory_error}
\end{figure}
Moreover, Figure~\ref{fig:coupled_oscillators_energy} shows a similar trend in energy conservation: the MLP exhibits unbounded drift up to \(10^{1}\), while all HNN variants maintain bounded oscillations around \(10^{-2}\), with GeoHNN again showing the most consistent behavior.
\begin{figure}[H]
    \centering
    \includegraphics[width=0.65\columnwidth]{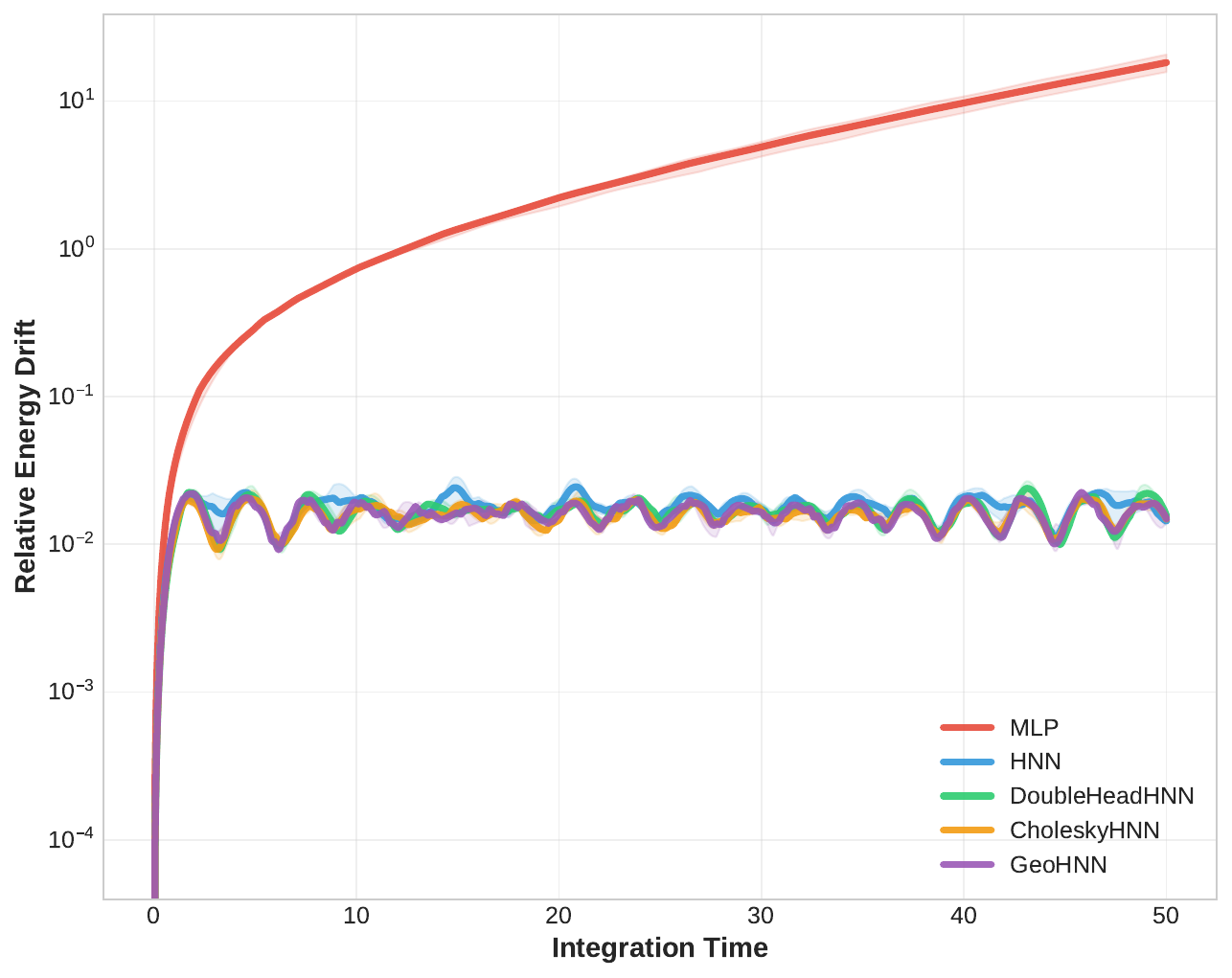}
    \caption{Energy drift as a function of time for $3$-coupled oscillators system.}
    \label{fig:coupled_oscillators_energy}
\end{figure}
\paragraph{Two Body Problem}
As shown in Figure~\ref{fig:two_body_trajectory_error}, the MLP baseline exhibits rapid error accumulation, with trajectory error reaching \(10^{2}\), and correspondingly the energy drift grows to \(10^{3}\) (see Figure~\ref{fig:two_body_energy}). Among structure-preserving models, HNN maintains stable trajectory errors around \(10^{1}\) and energy drift near \(10^{2}\), while CholeskyHNN shows similar performance but with slightly increased long-term error growth. GeoHNN demonstrates the most robust behavior, with trajectory error remaining near unity and energy drift bounded around \(10^{-1}\), highlighting the importance of properly accounting for the geometry of the inertia matrix.

\begin{figure}[H]
    \centering
    \includegraphics[width=0.65\columnwidth]{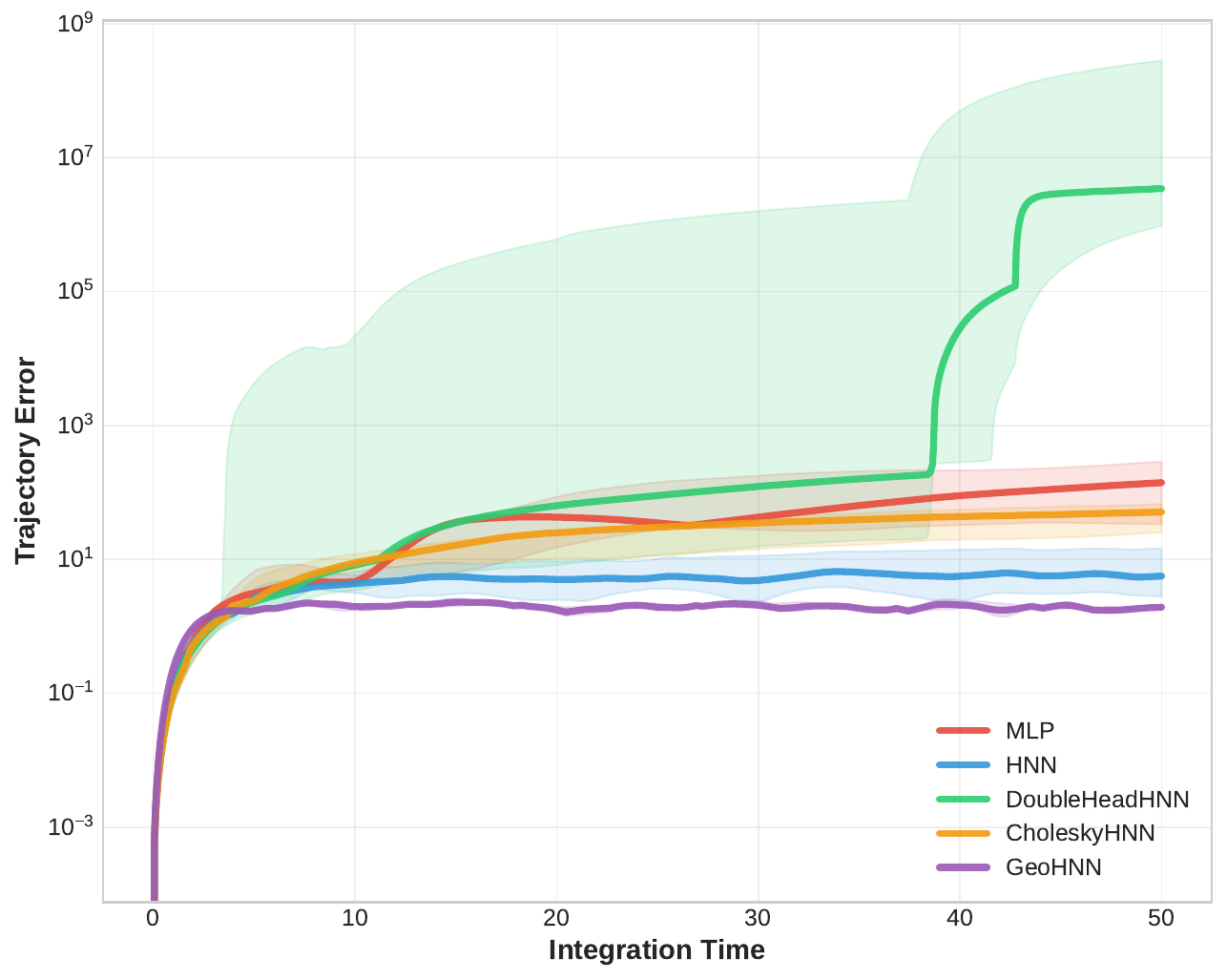}
    \caption{Long-term trajectory error as a function of time for the two body problem.}
    \label{fig:two_body_trajectory_error}
\end{figure}

In contrast, DoubleHeadHNN suffers severe numerical instability around \(t=40\), with trajectory error and energy drift diverging dramatically to \(10^{6}\) and \(10^{17}\), respectively. This instability arises because both CholeskyHNN and DoubleHeadHNN enforce positive definiteness algebraically but ignore the underlying SPD manifold geometry during optimization. GeoHNN, by explicitly respecting the SPD geometry through Riemannian optimization, inherently preserves the positive definiteness constraint, resulting in better numerical stability and accuracy.

\begin{figure}[H]
    \centering
    \includegraphics[width=0.65\columnwidth]{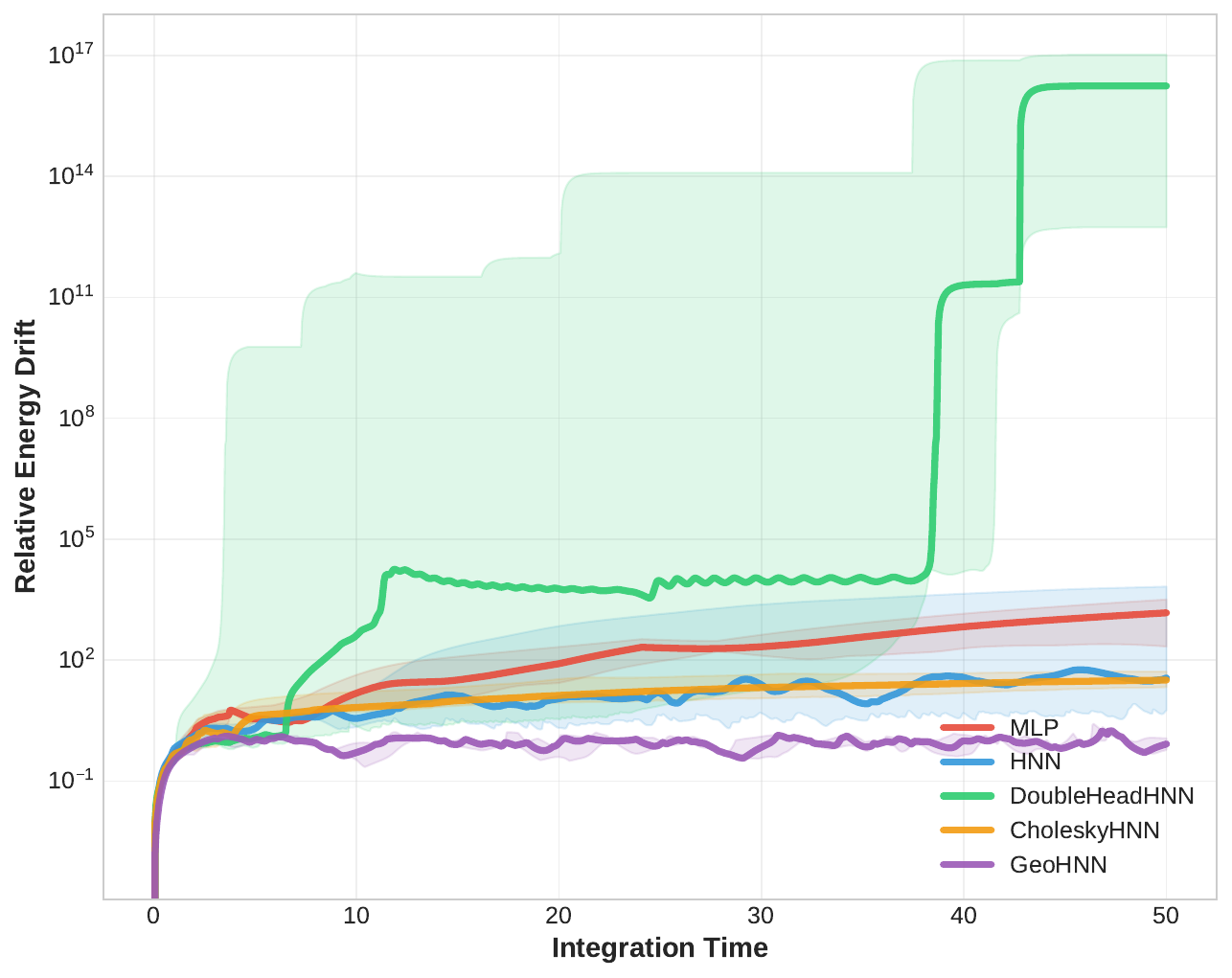}
    \caption{Energy Drift as a function of time for the two body problem.}
    \label{fig:two_body_energy}
\end{figure}

\paragraph{Single Pendulum}
As shown in Figure~\ref{fig:single_pendulum_trajectory_error}, the standard MLP suffers from severe instability, with trajectory errors growing exponentially from \(10^{-1}\) to over \(10^{2}\). In contrast, structure-preserving Hamiltonian variants exhibit markedly improved stability: the baseline HNN accumulates errors up to approximately 20; DoubleHeadHNN and CholeskyHNN improve stability with errors plateauing near 10; and GeoHNN achieves the best performance, maintaining errors near 5 throughout the integration horizon.
\begin{figure}[H]
    \centering
    \includegraphics[width=0.65\columnwidth]{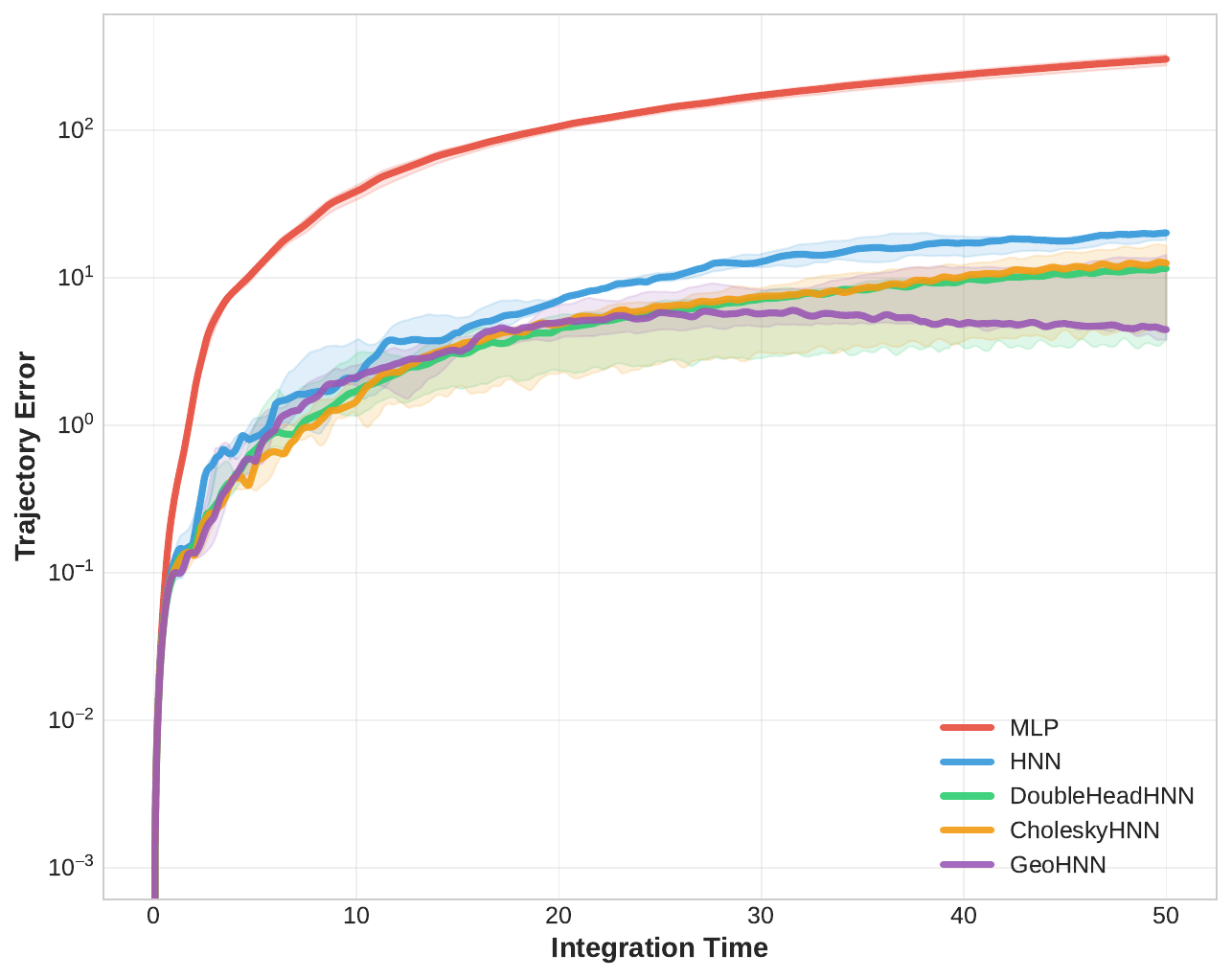}
    \caption{Long-term trajectory error as a function of  time for the single pendulum system.}
    \label{fig:single_pendulum_trajectory_error}
\end{figure}

Figure~\ref{fig:single_pendulum_energy} presents the corresponding energy drift results, where the MLP’s energy error increaase from 1 to \(10^{3}\), indicating a failure to preserve physical conservation laws. Conversely, all Hamiltonian models maintain energy drift bounded around \(10^{-1}\), representing over three orders of magnitude improvement in conservation accuracy.

\begin{figure}[H]
    \centering
    \includegraphics[width=0.65\columnwidth]{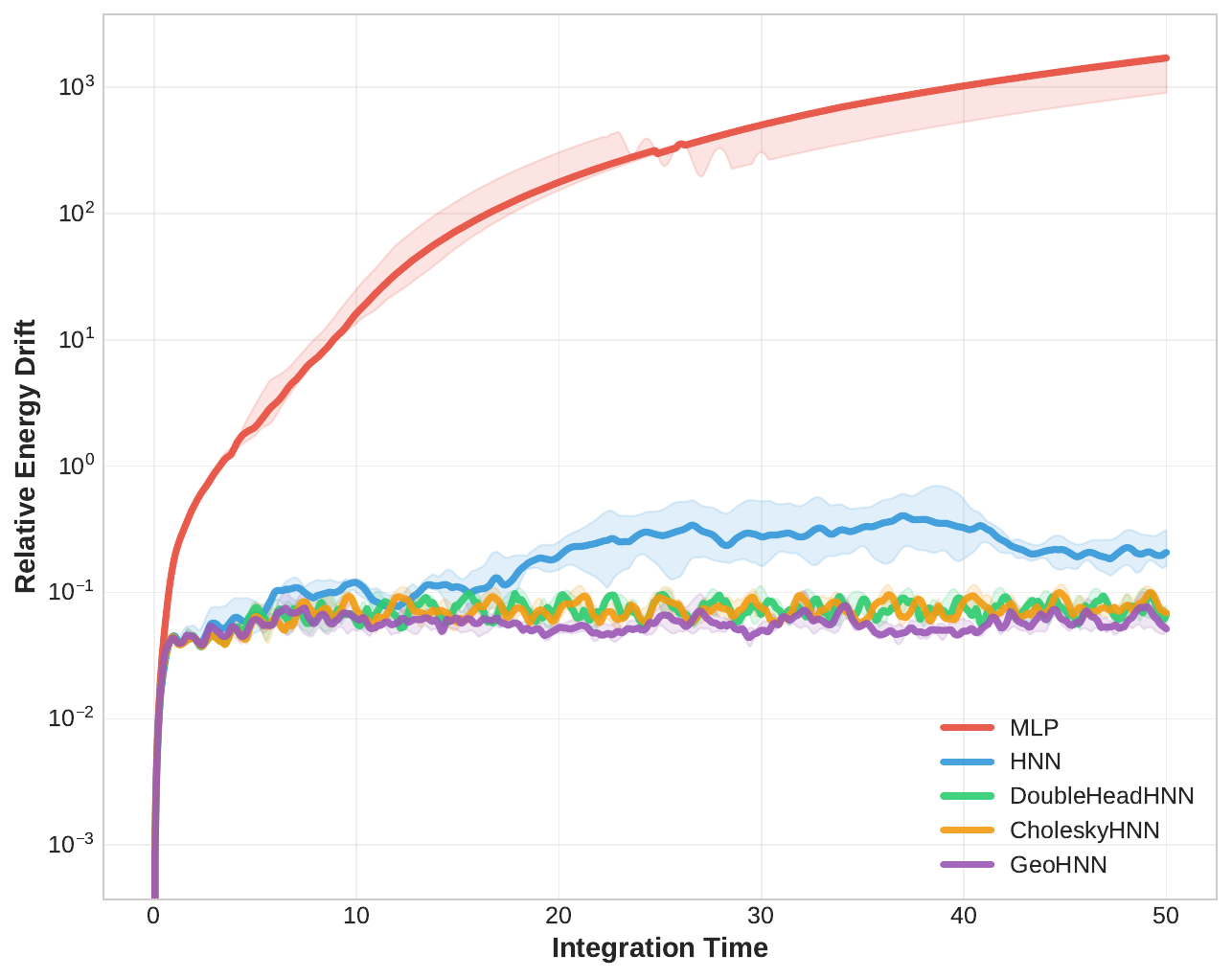}
    \caption{Energy drift as a function of time for the single pendulum system.}
    \label{fig:single_pendulum_energy}
\end{figure}

\subsubsection{High Dimensional Setting}
In the high-dimensional setting, we fix the base model to GeoHNN and vary only the model reduction component to isolate the impact of incorporating symplectic geometry. 

\paragraph{Dynamics Prediction}
We quantitatively compare the dynamics prediction accuracy of the proposed \emph{constrained biorthogonal autoencoder} against a \emph{vanilla unconstrained autoencoder}, focusing on position and momentum across six test trajectories and up to 501 degrees of freedom.

\begin{itemize}
    \item \textbf{Constrained Biorthogonal Autoencoder (Figure \ref{fig:ro-hnn-dynamics})} The constrained model achieves a mean position error of \(1.84 \times 10^{-1}\) (standard deviation: \(1.13 \times 10^{-1}\)) and a mean momentum error of \(1.05\) (standard deviation: \(8.31 \times 10^{-1}\)). Performance remains consistent across all degrees of freedom (DoF), indicating robust generalization to high-dimensional setting.
    \begin{figure}[H]
    \centering
    \includegraphics[width=\linewidth]{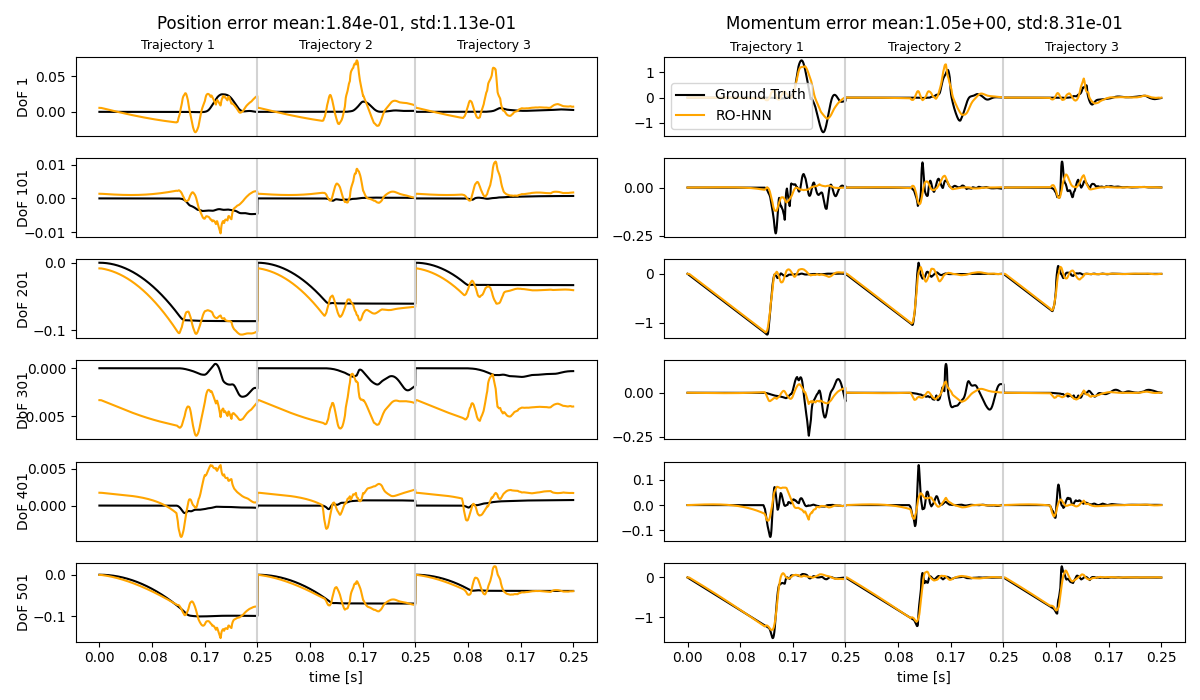}
    \caption{Prediction error across degrees of freedom for position and momentum using GeoHNN with a constrained biorthogonal autoencoder.}
    \label{fig:ro-hnn-dynamics}
\end{figure}
    \item \textbf{Vanilla Autoencoder (Figure \ref{fig:ro-hnn-euc-dynamics})} In contrast, the unconstrained model yields substantially worse accuracy, with a mean position error of \(1.82\) (standard deviation: \(2.09 \times 10^{-1}\)) and a mean momentum error of \(6.26\) (standard deviation: \(4.89\)). Prediction degrades significantly in high DoFs, where the model frequently underestimates dynamics, producing near-zero outputs when substantial motion is present.
    
\begin{figure}[H]
    \centering
    \includegraphics[width=\linewidth]{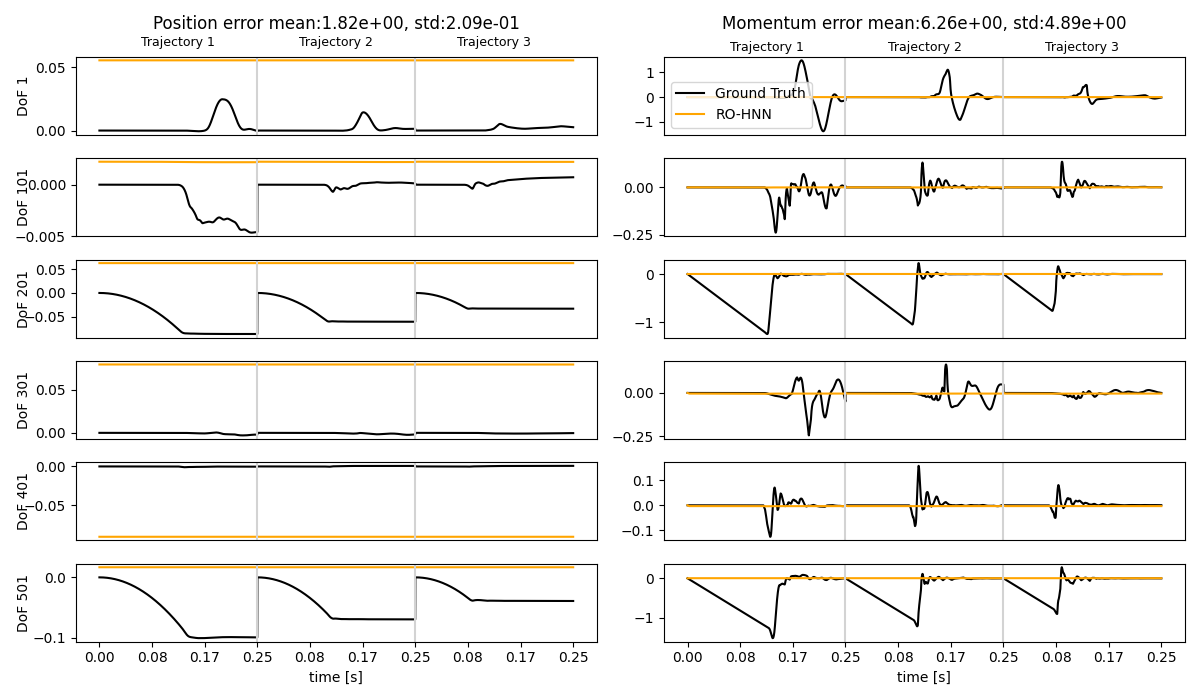}
    \caption{Prediction error across degrees of freedom for position and momentum using GeoHNN with a a unconstrained vanilla autoencoder.}
    \label{fig:ro-hnn-euc-dynamics}
\end{figure}
\end{itemize}
GeoHNN with constrained biorthogonal autoencoder improves position and momentum prediction by approximately \(10\times\) and \(6\times\), respectively, over the unconstrained baseline. It also exhibits markedly reduced variance. These improvements highlight the benefit of embedding geometric structure into the autoencoder. Specifically, the preservation of the symplectic structure appears essential in ensuring physically plausible and long-term stable dynamics, as further evidenced by sustained accuracy across all evaluated trajectories.

\paragraph{Reduction Map Property Satisfaction} We assess the fidelity of the learned reduction maps by measuring trajectory reconstruction error across all degrees of freedom, providing insight into how well each autoencoder preserves essential dynamical information.

\begin{itemize}
    \item \textbf{Constrained Biorthogonal Autoencoder (Figure \ref{fig:ro-hnn-ae})} The constrained model exhibits outstanding reduction map property satisfaction, with consistently low reconstruction error. Position reconstruction yields a mean absolute error of \(7.79 \times 10^{-2}\) (standard deviation: \(5.00 \times 10^{-3}\)), while momentum reconstruction achieves a mean error of \(8.75 \times 10^{-1}\) (std: \(7.83 \times 10^{-1}\)). Visual inspection confirms close alignment between reconstructed and ground truth trajectories across all DoFs, indicating that the learned latent space preserves the geometric structure of the original phase space.
    \begin{figure}[H]
    \centering
    \includegraphics[width=\linewidth]{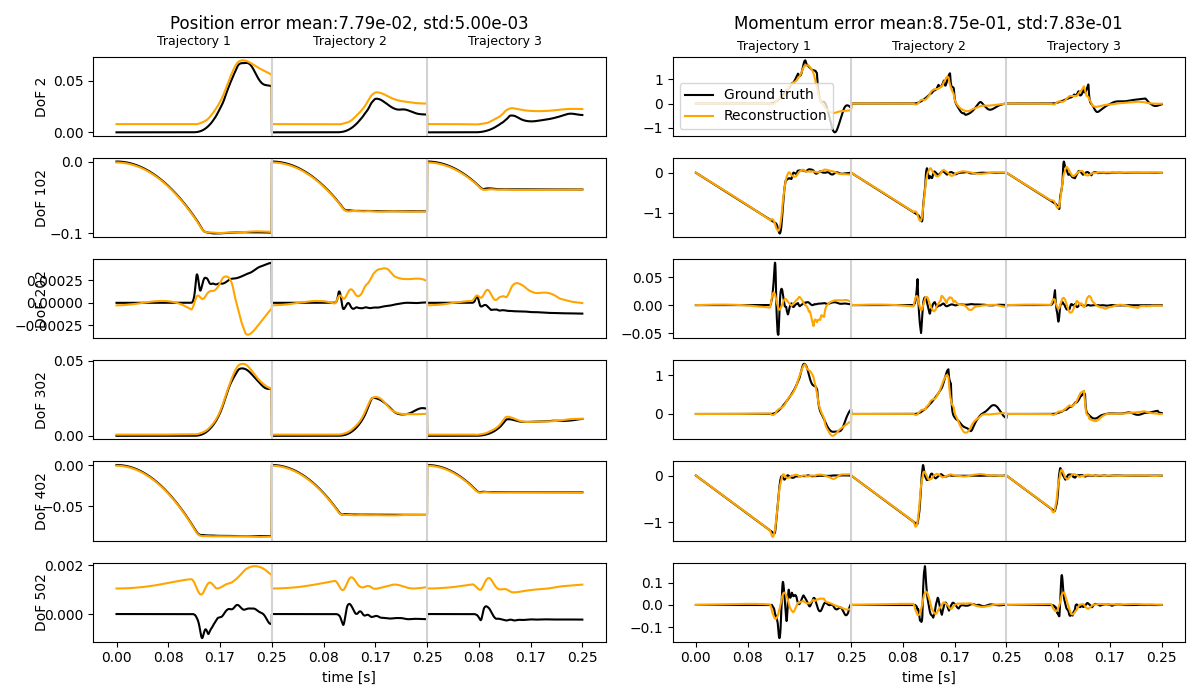}
    \caption{Trajectory reconstruction comparison for position and momentum using constrained biorthogonal autoencoder.}
    \label{fig:ro-hnn-ae}
\end{figure}
    \item \textbf{Vanilla Autoencoder (Figure \ref{fig:ro-hnn-euc-ae})} In contrast, the unconstrained model shows substantial degradation. Position reconstruction errors rise more than an order of magnitude to \(1.82\) (std: \(2.09 \times 10^{-1}\)), and momentum reconstruction fails catastrophically with mean error reaching \(6.26\) (std: \(4.89\)). Reconstruction is particularly poor in higher DoFs (e.g., DoF-302, DoF-402, DoF-502), where the model often collapses to near-constant outputs, failing to recover dynamic variability.

    \begin{figure}[H]
    \centering
    \includegraphics[width=\linewidth]{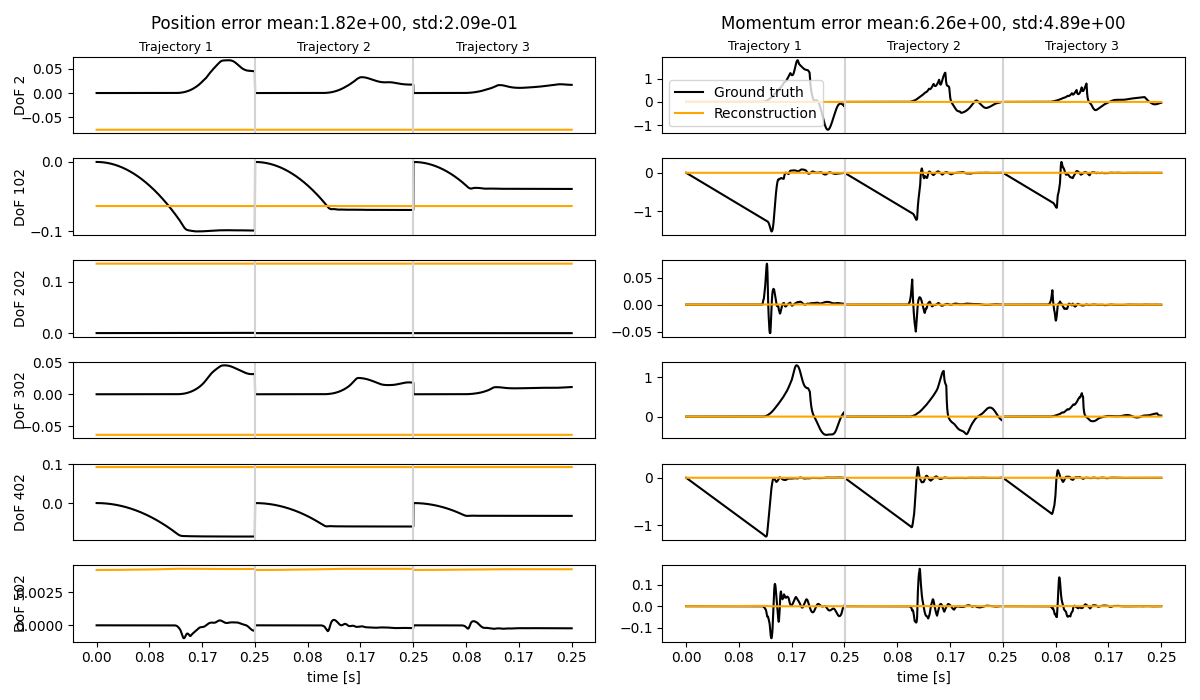}
    \caption{Trajectory reconstruction comparison for position and momentum using unconstrained vanilla autoencoder.}
    \label{fig:ro-hnn-euc-ae}
\end{figure}

\end{itemize}
The biorthogonal constrained autoencoder improves position and momentum reconstruction by approximately \(23\times\) and \(7\times\), respectively, compared to the unconstrained baseline. The significantly lower standard deviations highlight its consistency across trajectory segments. These results confirm that the constrained architecture learns a well-posed, geometry-preserving reduction map, while the unconstrained model suffers from severe information loss, likely due to lack of structural priors enforcing invertibility and symplecticity.

\section{Conclusion and Future Work}
In this work, we have developed a machine learning framework, GeoHNN, that successfully learns the dynamics of complex Hamiltonian systems by fundamentally respecting their geometric nature. By enforcing the Riemannian geometry of system inertia and the symplectic structure of phase space, our approach ensures that the learned models adhere to critical conservation laws, exhibit long-term stability, and generalize effectively across a variety of physical systems. This demonstrates that the explicit encoding of geometric priors is a crucial step beyond current physics-informed methods, reducing physical inconsistencies that have limited their applicability.

The success of this geometry-aware framework opens several exciting frontiers for data-driven physical modeling. The principles established here could be extended to a broader class of physical phenomena that have so far remained challenging. One such frontier is the dynamics of controlled and dissipative systems, where energy is no longer conserved but flows in a structured manner. This could directly apply to robotics, for example, by creating better models for complex tasks like controlling a robotic arm with flexible joints or helping a legged robot walk by understanding dynamic contact and balance. Our method is especially suited for these problems because it correctly models the inertia of rigid-body systems often found in robotics. Another area for future work lies in infinite-dimensional systems, such as fluid and soft-body dynamics, which are governed by partial differential equations and represent a significant jump in complexity. Finally, integrating these geometric constraints into models of stochastic Hamiltonian systems could reveal new insights into molecular dynamics and climate modeling, where noise and structure-preserving dynamics might be linked. This approach could even be extended to financial markets, where recent work proposes a Hamiltonian structure by defining market momentum as signed trading volume and position as the asset’s price; applying our geometry-aware model could lead to new methods for optimal trade execution.

By bridging data-driven modeling with the foundational geometric principles of physics, this work opens the door for a new generation of computational models capable of learning and predicting the dynamics of the physical world with improved fidelity and robustness.

\bibliographystyle{plain}
\bibliography{references} 

\begin{thebibliography}{10}

\bibitem{arnold1989mathematical}
V.I. Arnold.
\newblock {\em Mathematical methods of classical mechanics}, volume~60.
\newblock Springer, 1989.

\bibitem{bank2021autoencoders}
Dor Bank, Noam Koenigstein, and Raja Giryes.
\newblock Autoencoders, 2021.

\bibitem{inductive_biases}
Peter~W. Battaglia, Jessica~B. Hamrick, Victor Bapst, Alvaro Sanchez-Gonzalez, Vinicius Zambaldi, Mateusz Malinowski, Andrea Tacchetti, David Raposo, Adam Santoro, Ryan Faulkner, Caglar Gulcehre, Francis Song, Andrew Ballard, Justin Gilmer, George Dahl, Ashish Vaswani, Kelsey Allen, Charles Nash, Victoria Langston, Chris Dyer, Nicolas Heess, Daan Wierstra, Pushmeet Kohli, Matt Botvinick, Oriol Vinyals, Yujia Li, and Razvan Pascanu.
\newblock Relational inductive biases, deep learning, and graph networks, 2018.

\bibitem{berkooz_pod}
G~Berkooz, P~Holmes, and J~L Lumley.
\newblock The proper orthogonal decomposition in the analysis of turbulent flows.
\newblock {\em Annual Review of Fluid Mechanics}, 25(Volume 25, 1993):539--575, 1993.

\bibitem{buchfink2021}
Patrick Buchfink, Silke Glas, and Bernard Haasdonk.
\newblock Symplectic model reduction of hamiltonian systems on nonlinear manifolds, 2021.

\bibitem{buchfink2024modelreductionmanifoldsdifferential}
Patrick Buchfink, Silke Glas, Bernard Haasdonk, and Benjamin Unger.
\newblock Model reduction on manifolds: A differential geometric framework, 2024.

\bibitem{buchfink2024}
Patrick Buchfink, Silke Glas, Bernard Haasdonk, and Benjamin Unger.
\newblock Model reduction on manifolds: A differential geometric framework, 2024.

\bibitem{bécigneul2019riemannianadaptiveoptimizationmethods}
Gary Bécigneul and Octavian-Eugen Ganea.
\newblock Riemannian adaptive optimization methods, 2019.

\bibitem{governingequationsdiscovery}
Kathleen Champion, Bethany Lusch, J.~Nathan Kutz, and Steven~L. Brunton.
\newblock Data-driven discovery of coordinates and governing equations.
\newblock {\em Proceedings of the National Academy of Sciences}, 116(45):22445–22451, October 2019.

\bibitem{chen2019neuralordinarydifferentialequations}
Ricky T.~Q. Chen, Yulia Rubanova, Jesse Bettencourt, and David Duvenaud.
\newblock Neural ordinary differential equations, 2019.

\bibitem{symplecticrecurrentneuralnetworks}
Zhengdao Chen, Jianyu Zhang, Martin Arjovsky, and Léon Bottou.
\newblock Symplectic recurrent neural networks, 2020.

\bibitem{generalizedhamiltoninas}
Harsh Choudhary, Chandan Gupta, Vyacheslav kungrutsev, Melvin Leok, and Georgios Korpas.
\newblock Learning generalized hamiltonians using fully symplectic mappings, 2024.

\bibitem{lnns}
Miles Cranmer, Sam Greydanus, Stephan Hoyer, Peter Battaglia, David Spergel, and Shirley Ho.
\newblock Lagrangian neural networks, 2020.

\bibitem{côte2024hamiltonianreductionusingconvolutional}
Raphaël Côte, Emmanuel Franck, Laurent Navoret, Guillaume Steimer, and Vincent Vigon.
\newblock Hamiltonian reduction using a convolutional auto-encoder coupled to an hamiltonian neural network, 2024.

\bibitem{feragen2017geodesic}
Aasa Feragen and Andrea Fuster.
\newblock {\em Geometries and Interpolations for Symmetric Positive Definite Matrices}, pages 85--113.
\newblock Springer, 01 2017.

\bibitem{simple_hnn_lnn}
Marc Finzi, Ke~Wang, and Andrew Wilson.
\newblock Simplifying hamiltonian and lagrangian neural networks via explicit constraints, 10 2020.

\bibitem{molecular_dynamics}
Daan Frenkel and Berend Smit.
\newblock {\em Understanding Molecular Simulation: From Algorithms to Applications}, volume~1 of {\em Computational Science Series}.
\newblock Academic Press, San Diego, second edition, 2002.

\bibitem{friedl2025riemannianframeworklearningreducedorder}
Katharina Friedl, Noémie Jaquier, Jens Lundell, Tamim Asfour, and Danica Kragic.
\newblock A riemannian framework for learning reduced-order lagrangian dynamics, 2025.

\bibitem{goldstein:mechanics}
Herbert Goldstein.
\newblock {\em Classical Mechanics}.
\newblock Addison-Wesley, 1980.

\bibitem{hamiltonian_neural_networks}
Sam Greydanus, Misko Dzamba, and Jason Yosinski.
\newblock Hamiltonian neural networks, 2019.

\bibitem{grinfeld2013introduction}
Pavel Grinfeld.
\newblock {\em Introduction to Tensor Analysis and the Calculus of Moving Surfaces}.
\newblock Springer Science \& Business Media, illustrated edition, 2013.

\bibitem{Isham1989ModernDG}
C.~J. Isham.
\newblock {\em Modern Differential Geometry For Physicists}.
\newblock World Scientific, 1989.

\bibitem{geoopt2020kochurov}
Max Kochurov, Rasul Karimov, and Serge Kozlukov.
\newblock Geoopt: Riemannian optimization in pytorch, 2020.

\bibitem{Lanczos1970}
Cornelius Lanczos.
\newblock {\em The variational principles of mechanics}.
\newblock Dover books on physics and chemistry; no. 4. Dover Publications, New York, fourth edition edition, 1970.

\bibitem{lecun2015deep}
Yann LeCun, Yoshua Bengio, and Geoffrey Hinton.
\newblock Deep learning.
\newblock {\em nature}, 521(7553):436, 2015.

\bibitem{lee_smooth_manifolds}
John~M. Lee.
\newblock {\em Introduction to Smooth Manifolds}.
\newblock Springer, 2000.

\bibitem{lin2019riemannian}
Zhenhua Lin.
\newblock Riemannian geometry of symmetric positive definite matrices via cholesky decomposition.
\newblock {\em SIAM Journal on Matrix Analysis and Applications}, 40(4):1353–1370, January 2019.

\bibitem{Liouville1838}
J.~Liouville.
\newblock Note sur la théorie de la variation des constantes arbitraires.
\newblock {\em Journal de Mathématiques Pures et Appliquées}, pages 342--349, 1838.

\bibitem{loshchilov2019decoupledweightdecayregularization}
Ilya Loshchilov and Frank Hutter.
\newblock Decoupled weight decay regularization, 2019.

\bibitem{deepleangragiaannetworks}
Michael Lutter, Christian Ritter, and Jan Peters.
\newblock Deep lagrangian networks: Using physics as model prior for deep learning, 2019.

\bibitem{need_for_biases}
Tom Mitchell.
\newblock The need for biases in learning generalizations, 10 2002.

\bibitem{spd_manifold_metric}
Maher Moakher.
\newblock A differential geometric approach to the geometric mean of symmetric positive-definite matrices.
\newblock {\em SIAM J. Matrix Anal. Appl.}, 26(3):735–747, March 2005.

\bibitem{robotics_systems}
Richard~M. Murray, S.~Shankar Sastry, and Li~Zexiang.
\newblock {\em A Mathematical Introduction to Robotic Manipulation}.
\newblock CRC Press, Inc., USA, 1st edition, 1994.

\bibitem{Noether1918}
E.~Noether.
\newblock Invariante variationsprobleme.
\newblock {\em Nachrichten von der Gesellschaft der Wissenschaften zu Göttingen, Mathematisch-Physikalische Klasse}, 1918:235--257, 1918.

\bibitem{otto2023}
Samuel~E. Otto, Gregory~R. Macchio, and Clarence~W. Rowley.
\newblock Learning nonlinear projections for reduced-order modeling of dynamical systems using constrained autoencoders, 2023.

\bibitem{peng2015symplecticmodelreductionhamiltonian}
Liqian Peng and Kamran Mohseni.
\newblock Symplectic model reduction of hamiltonian systems, 2015.

\bibitem{model_order_reduction}
Wil Schilders, Henk Van~der Vorst, and Joost Rommes.
\newblock {\em Model Order Reduction: Theory, Research Aspects and Applications}, volume~13.
\newblock Springer, 01 2008.

\bibitem{tu2020differential}
Loring~W. Tu.
\newblock {\em Differential Equations on Manifolds}.
\newblock Universitext. Springer, 2020.

\bibitem{venkat2021convolutionalautoencodersreducedordermodeling}
Sreeram Venkat, Ralph~C. Smith, and Carl~T. Kelley.
\newblock Convolutional autoencoders for reduced-order modeling, 2021.

\bibitem{walter1998ordinary}
Wolfgang Walter.
\newblock {\em Ordinary Differential Equations}.
\newblock Springer, 1998.

\bibitem{xiao2024generalizedlagrangianneuralnetworks}
Shanshan Xiao, Jiawei Zhang, and Yifa Tang.
\newblock Generalized lagrangian neural networks, 2024.

\bibitem{zhong2024symplecticodenetlearninghamiltonian}
Yaofeng~Desmond Zhong, Biswadip Dey, and Amit Chakraborty.
\newblock Symplectic ode-net: Learning hamiltonian dynamics with control, 2024.

\end{thebibliography}
\newpage
\appendix
\section{Mathematical Formalism}

In this section, we outline the mathematical background needed to understand the theoretical aspects of our framework. Since our goal is to preserve the underlying geometry of the physical systems we model, we work with ``manifolds'', which generalize Euclidean spaces and allow us to represent more complex geometric structures.

\subsection{A short introduction to Smooth Manifolds}
In many physical systems, the space of possible configurations is not flat like $\mathbb{R}^n$, but curved or constrained in some way. To capture this, we model such spaces as manifolds. Intuitively, a manifold is a space that may be curved globally but looks like a flat Euclidean space locally.

\begin{definition}[Topological Manifold]
A Hausdorff topological space $M$ is called an $n$-dimensional topological manifold if:
\begin{enumerate}
\item $M$ is second-countable (possesses a countable basis for its topology)
\item $M$ is locally Euclidean: for each point $p \in M$, there exists an open neighborhood $U \ni p$ and a homeomorphism $\varphi: U \to V \subseteq \R^n$, where $V$ is open in $\R^n$.
\end{enumerate}

The pair $(U, \varphi)$ is called a coordinate chart or simply a chart. The homeomorphism $\varphi$ provides local coordinates for points in $U$.
\end{definition}

\begin{remark}
The Hausdorff condition ensures that distinct points can be separated by disjoint open neighborhoods, while second-countability guarantees the existence of partitions of unity, crucial for many constructions.
\end{remark}

To perform calculus on manifolds, we need additional structure beyond topology.

\begin{definition}[Smooth Atlas]
A smooth atlas on a topological manifold $M$ is a collection $\mathcal{A} = \{(U_\alpha, \varphi_\alpha)\}_{\alpha \in I}$ of charts such that:
\begin{enumerate}
\item $\{U_\alpha\}_{\alpha \in I}$ covers $M$
\item For any two charts $(U_\alpha, \varphi_\alpha)$ and $(U_\beta, \varphi_\beta)$ with $U_\alpha \cap U_\beta \neq \emptyset$, the transition map
\[\varphi_\beta \circ \varphi_\alpha^{-1}: \varphi_\alpha(U_\alpha \cap U_\beta) \to \varphi_\beta(U_\alpha \cap U_\beta)\]
is a smooth ($C^\infty$) diffeomorphism.
\end{enumerate}
\end{definition}

\begin{definition}[Smooth Manifold]
A smooth manifold is a topological manifold $M$ equipped with a maximal smooth atlas.
\end{definition}

\begin{theorem}[Existence of Maximal Atlas]
Every smooth atlas $\mathcal{A}$ on $M$ is contained in a unique maximal smooth atlas $\mathcal{A}_{\max}$.
\end{theorem}

\begin{proof}
Define $\mathcal{A}_{\max}$ as the collection of all charts compatible with every chart in $\mathcal{A}$. Verification of maximality follows from the definition.
\end{proof}

\begin{example}
In classical mechanics, the configuration space of a physical system is the set of all possible positions the system can attain. When this space has a smooth structure, it is modeled as a configuration manifold, denoted by $\mathcal{Q}$.

\begin{enumerate}
    \item For a single pendulum, the position is described by an angle $\theta$, which lives on the circle $S^1$, a 1-dimensional manifold.
    \item For a double pendulum, the configuration space is $S^1 \times S^1$, a 2-dimensional manifold representing two angular degrees of freedom.
    \item For a rigid body in 3D space, the configuration manifold is $SE(3)$, the special Euclidean group of rigid transformations, which combines translation ($\mathbb{R}^3$) and rotation ($SO(3)$).
\end{enumerate}
These examples illustrate that configuration spaces naturally form smooth manifolds, often with nontrivial topology or curvature, making the manifold framework essential for modeling dynamics accurately.
\end{example}

Having defined smooth manifolds, we now turn to functions between them. Many physical quantities of interest, such as the Lagrangian and Hamiltonian, are best understood as smooth maps between manifolds. This allows us to reason about their differentiability, composition, and geometric properties in a coordinate-free manner.

\begin{definition}[Smooth Map]
A map $f: M \to N$ between smooth manifolds is smooth if for every pair of charts $(U, \varphi)$ on $M$ and $(V, \psi)$ on $N$ with $f(U) \cap V \neq \emptyset$, the composition
\[\psi \circ f \circ \varphi^{-1}: \varphi(U \cap f^{-1}(V)) \to \psi(f(U) \cap V)\]
is smooth as a map between Euclidean spaces.
\end{definition}
This ensures that the map $f$ is differentiable in every coordinate system and therefore behaves smoothly with respect to the underlying manifold structures.
\begin{definition}[Diffeomorphism]
A smooth bijection $f: M \to N$ with a smooth inverse is called a diffeomorphism. Manifolds related by a diffeomorphism are said to be diffeomorphic.
\end{definition}

In our context, diffeomorphisms play a crucial role in modeling coordinate-independent dynamics. For instance, canonical transformations in Hamiltonian mechanics, such as the changes of generalized coordinates, are diffeomorphisms between phase spaces that preserve the geometric structure of the equations of motion. These transformations allow us to express the same physical dynamics in different coordinate systems without altering the underlying mechanics.

Now that we have introduced smooth manifolds and smooth maps between them, we turn to the differential calculus necessary for modeling physical systems. To define derivatives, flows, and vector fields on manifolds, we need a notion of tangent vectors. Although subspaces of $\R^n$ inherit a notion of tangent vectors from the ambient space, abstract manifolds require an intrinsic formulation.
\begin{definition}[Tangent Vector as Derivation]
Let $M$ be a smooth manifold and $p \in M$. A tangent vector at $p$ is a linear map $v: C^\infty(M) \to \R$ satisfying the Leibniz rule:
\[v(fg) = f(p)v(g) + g(p)v(f)\]
for all $f, g \in C^\infty(M)$.

The set of all tangent vectors at $p$ forms a vector space denoted $T_p M$, called the tangent space at $p$.
\end{definition}

\begin{theorem}[Dimension of Tangent Space]
If $M$ is an $n$-dimensional manifold, then $\dim(T_p M) = n$ for all $p \in M$.
\end{theorem}

\begin{proof}
Let $(U, \varphi)$ be a chart with $p \in U$ and $\varphi(p) = 0$. Define coordinate vectors
\[\frac{\partial}{\partial x^i}\bigg|_p : f \mapsto \frac{\partial(f \circ \varphi^{-1})}{\partial x^i}\bigg|_{\varphi(p)}\]

We claim that $\{\frac{\partial}{\partial x^1}|_p, \ldots, \frac{\partial}{\partial x^n}|_p\}$ forms a basis for $T_p M$.

\textit{Linear Independence:} Suppose $\sum_i a^i \frac{\partial}{\partial x^i}|_p = 0$. Applying this to the coordinate function $x^j = \pi^j \circ \varphi$ gives $a^j = 0$.

\textit{Spanning:} For any $v \in T_p M$, define $v^i = v(x^i)$. We show that $v = \sum_i v^i \frac{\partial}{\partial x^i}|_p$ using Taylor's theorem and the Leibniz rule.
\end{proof}

Alternatively, a more geometric interpretation defines tangent vectors as velocities of curves passing through the point.

\begin{definition}[Velocity Vector]
Let $\gamma: (-\varepsilon, \varepsilon) \to M$ be a smooth curve with $\gamma(0) = p$. The velocity vector of $\gamma$ at $p$ is the tangent vector $\dot{\gamma}(0) \in T_p M$ defined by
\[\dot{\gamma}(0)(f) = \frac{d}{dt}(f \circ \gamma)\bigg|_{t=0}\]
\end{definition}

\begin{theorem}[Curve Characterization]
Every tangent vector arises as the velocity vector of some curve. 
\end{theorem}

\begin{remark}
Strictly speaking, the tangent vector is determined by the equivalence class of curves that share the same first-order behavior at $p$.
\end{remark}
To organize all tangent spaces into a single geometric object, we define the tangent bundle.

\begin{definition}[Tangent Bundle]
The tangent bundle of $M$ is
\[TM = \bigsqcup_{p \in M} T_p M = \{(p, v) : p \in M, v \in T_p M\}\]
with natural projection $\pi: TM \to M$ given by $\pi(p, v) = p$.
\end{definition}

\begin{theorem}[Tangent Bundle Structure]
$TM$ has a natural structure as a smooth manifold of dimension $2n$, where $n = \dim(M)$.
\end{theorem}

\begin{proof}[Proof Construction]
For each chart $(U, \varphi)$ on $M$, define
\begin{align}
\Psi: \pi^{-1}(U) &\to \varphi(U) \times \R^n\\
\Psi(p, v) &= \left(\varphi(p), \sum_{i=1}^n v^i \frac{\partial}{\partial x^i}\bigg|_p\right)
\end{align}
These provide charts making $TM$ into a smooth manifold.
\end{proof}

\begin{example}
Let $M = \mathcal{Q}$ be the configuration manifold of a mechanical system. A Lagrangian is a smooth function $L: T\mathcal{Q} \to \mathbb{R}$ that assigns a scalar (typically kinetic minus potential energy) to each configuration and velocity. The Euler–Lagrange equations derived from $L$ govern the dynamics of the system, and are naturally formulated on the tangent bundle $T\mathcal{Q}$.
\end{example}

To describe the dynamics of a system evolving on a manifold, we need to specify how a point moves at each location. This leads to the concept of a vector field, which assigns a tangent vector to every point on the manifold, intuitively representing the direction and speed of motion of a particle at each location.

\begin{definition}[Vector Field]
A vector field on $M$ is a smooth section of the projection $\pi: TM \to M$, i.e., a smooth map $X: M \to TM$ such that $\pi \circ X = \id_M$. This means that for each point $p \in M$, the vector $X(p)$ lies in the tangent space $T_p M$.
We denote the space of smooth vector fields by $\mathfrak{X}(M)$.

\end{definition}

\begin{example}[Coordinate Vector Fields]
Let $(U, \varphi = (x^1, \dots, x^n))$ be a coordinate chart on $M$. Then the coordinate vector fields
\begin{equation*}
    \frac{\partial}{\partial x^1}, \dots, \frac{\partial}{\partial x^n}
\end{equation*}
form a local basis for the tangent bundle $TM$ over $U$. That is, any smooth vector field $X \in \mathfrak{X}(U)$ can be expressed locally as
\begin{equation*}
X = \sum_{i=1}^n X^i(x) \frac{\partial}{\partial x^i},
\end{equation*}
where each $X^i: U \to \mathbb{R}$ is a smooth function, called the component of $X$ in the $i$-th coordinate direction.
\end{example}

Having defined a vector field $X \in \mathfrak{X}(M)$ that governs how points move over time, a natural question arises: given an initial point $p \in M$, how does it evolve under the flow of $X$?

Intuitively, an \textit{integral curve} of $X$ is a trajectory that follows the direction specified by the vector field at every point. It represents the motion of a particle whose instantaneous velocity at any time $t$ is given by $X$ evaluated at its current location.

\begin{definition}[Integral Curve]
Let $X \in \mathfrak{X}(M)$ be a smooth vector field. A smooth curve $\gamma: I \to M$, where $I \subseteq \mathbb{R}$ is an open interval, is called an \emph{integral curve} of $X$ if it satisfies
\[
\dot{\gamma}(t) = X(\gamma(t)) \quad \text{for all } t \in I.
\]
\end{definition}

This is a geometric generalization of ordinary differential equations (ODEs) to manifolds. The condition says that the velocity of the curve $\gamma$ at each time $t$ must agree with the vector field $X$ at the current point $\gamma(t)$.

\begin{definition}[Initial Value Problem (IVP)]
Given a point $p \in M$ and a vector field $X \in \mathfrak{X}(M)$, the initial value problem consists of finding an integral curve $\gamma: I \to M$ such that
\[
\dot{\gamma}(t) = X(\gamma(t)), \quad \gamma(0) = p.
\]
\end{definition}

Under mild regularity conditions, the existence and uniqueness of local solutions to this IVP are guaranteed by the classical Picard–Lindelöf theorem adapted to manifolds.

\begin{example}[Lagrangian Dynamics as an IVP]
Let $\mathcal{Q}$ be the configuration manifold of a mechanical system, and let $L: T\mathcal{Q} \to \mathbb{R}$ be a smooth Lagrangian function. The Euler–Lagrange equations for a curve $q(t) \in \mathcal{Q}$ are given by
\[
\frac{d}{dt} \left( \frac{\partial L}{\partial \dot{q}^i} \right) - \frac{\partial L}{\partial q^i} = 0.
\]

To express these dynamics as a first-order system on the tangent bundle $T\mathcal{Q}$, we define the state variable as $z = (q^i, \dot{q}^i)$, and construct a vector field $X_L \in \mathfrak{X}(T\mathcal{Q})$ such that
\[
X_L(z) = \left( \dot{q}^i, \ddot{q}^i \right),
\]
where the accelerations $\ddot{q}^i$ are implicitly defined by the Euler–Lagrange equations. More precisely, under regularity (i.e., non-degenerate mass matrix),
\[
\ddot{q}^i = \left( \frac{\partial^2 L}{\partial \dot{q}^i \partial \dot{q}^j} \right)^{-1} \left( \frac{\partial L}{\partial q^j} - \frac{\partial^2 L}{\partial \dot{q}^j \partial q^k} \dot{q}^k \right).
\]

The corresponding initial value problem is:
\[
\dot{\gamma}(t) = X_L(\gamma(t)), \quad \gamma(0) = (q_0, \dot{q}_0) \in T\mathcal{Q},
\]
where the solution $\gamma(t) \in T\mathcal{Q}$ gives the time evolution of both the position and velocity.
\end{example}

Although Lagrangian mechanics is naturally formulated on the tangent bundle $T\mathcal{Q}$, Hamiltonian mechanics resides on the cotangent bundle $T^*\mathcal{Q}$. This transition reflects a change of variables from velocities $\dot{q}$ to momenta $p$, introducing a dual geometric structure based on covectors.

\begin{definition}[Cotangent Space]
Let $M$ be a smooth manifold and $p \in M$. The cotangent space at $p$, denoted $T_p^* M$, is the dual vector space of $T_p M$:
\begin{equation}
    T_p^* M = \{\omega: T_p M \to \mathbb{R} \mid \omega \text{ is linear} \}.
\end{equation}

\end{definition}

\begin{definition}[Cotangent Bundle]
The cotangent bundle of $M$ is the disjoint union
\[
T^* M = \bigsqcup_{p \in M} T_p^* M,
\]
with a natural smooth manifold structure of dimension $2n$ if $\dim(M) = n$. It comes equipped with the canonical projection $\pi: T^*M \to M$.
\end{definition}

There is a natural dual pairing between tangent and cotangent vectors: if $v \in T_p M$ and $\omega \in T_p^* M$, then $\omega(v) \in \mathbb{R}$. In the context of Hamiltonian mechanics, this duality is formalized via the \emph{Legendre transform}, which enables the transition from Lagrangian to Hamiltonian formulations.

\begin{definition}[Legendre Transform]
Let $L: T\mathcal{Q} \to \mathbb{R}$ be a regular Lagrangian. The Legendre transform is a smooth map $\mathbb{F}L: T\mathcal{Q} \to T^*\mathcal{Q}$ defined by
\[
\mathbb{F}L(q, \dot{q}) = \left(q, \frac{\partial L}{\partial \dot{q}}\right),
\]
which maps generalized velocities to conjugate momenta.
\end{definition}

If $L$ is regular (i.e., the Hessian matrix $\frac{\partial^2 L}{\partial \dot{q}^i \partial \dot{q}^j}$ is invertible), then the Legendre transform is a local diffeomorphism. When it is globally invertible, one can define the Hamiltonian via
\[
H(q, p) = \left\langle p, \dot{q} \right\rangle - L(q, \dot{q}),
\]
where $(q, \dot{q}) = (\mathbb{F}L)^{-1}(q, p)$.

Moreover, the cotangent bundle $T^*\mathcal{Q}$, in this case, possesses a canonical \emph{symplectic structure}.

\begin{definition}[Symplectic Manifold]
A symplectic manifold is a pair $(M, \omega)$ where $M$ is a smooth manifold and $\omega$ is a closed, non-degenerate 2-form:
\[
d\omega = 0, \quad \text{and} \quad \forall v \neq 0 \in T_p M, \exists u \in T_p M \text{ s.t. } \omega_p(u, v) \neq 0.
\]
\end{definition}

\begin{example}[Canonical Symplectic Form]
The cotangent bundle $T^*\mathcal{Q}$ over a configuration manifold $\mathcal{Q}$ is naturally a symplectic manifold. It is equipped with the canonical 2-form
\[
\omega = \sum_i dq^i \wedge dp_i,
\]
which arises as the exterior derivative of the Liouville 1-form $\theta = \sum_i p_i \, dq^i$, i.e., $\omega = -d\theta$.
\end{example}

Now, given a smooth Hamiltonian function $H: T^*\mathcal{Q} \to \mathbb{R}$, the dynamics are encoded in the Hamiltonian vector field $X_H$, defined implicitly by
\[
\iota_{X_H} \omega = dH.
\]
This vector field generates the flow of the system in phase space, and its integral curves describe the time evolution of state trajectories.

\begin{example}[Hamiltonian Dynamics as an IVP]
Let $\mathcal{Q}$ be a configuration manifold and $H: T^*\mathcal{Q} \to \mathbb{R}$ a smooth Hamiltonian. In local coordinates $(q^i, p_i)$, the Hamiltonian vector field is given by:
\[
X_H = \frac{\partial H}{\partial p_i} \frac{\partial}{\partial q^i} - \frac{\partial H}{\partial q^i} \frac{\partial}{\partial p_i}.
\]

Let $\gamma(t) = (q^i(t), p_i(t))$ be an integral curve of $X_H$. Then the dynamics are governed by the initial value problem:
\[
\dot{q}^i = \frac{\partial H}{\partial p_i}, \quad \dot{p}_i = -\frac{\partial H}{\partial q^i}, \quad \gamma(0) = (q_0^i, p_i^0).
\]

This system of first-order ODEs describes the evolution of the system in phase space under the Hamiltonian flow.
\end{example}

Next, to extend optimization methods to curved spaces such as manifolds, we must first endow the manifold with additional geometric structure. This is achieved through a \emph{Riemannian metric}, which allows us to define inner products, gradients, and descent directions intrinsically on the manifold.

\begin{definition}[Riemannian Metric]
Let $M$ be a smooth manifold. A \emph{Riemannian metric} $g$ is a smoothly varying inner product on the tangent spaces of $M$. That is, for each $p \in M$, $g_p: T_p M \times T_p M \to \mathbb{R}$ is a symmetric, positive-definite bilinear form, and the map $p \mapsto g_p(u, v)$ is smooth for all smooth vector fields $u, v$.
\end{definition}

This allows us to measure lengths of tangent vectors and angles between them. It also enables us to define the notion of gradient in a way that respects the underlying manifold geometry.

\begin{definition}[Differential]
Let $f: M \to \mathbb{R}$ be a smooth function. The \emph{differential} of $f$ at $p \in M$ is the linear map
\[
df_p: T_p M \to \mathbb{R}, \quad df_p(v) = v[f] = \left.\frac{d}{dt} f(\gamma(t))\right|_{t=0},
\]
where $\gamma(t)$ is any smooth curve with $\gamma(0) = p$ and $\dot{\gamma}(0) = v$.
\end{definition}

The Riemannian metric gives a natural way to associate a vector to the differential of a function.

\begin{definition}[Riemannian Gradient]
Let $(M, g)$ be a Riemannian manifold and $f: M \to \mathbb{R}$ a smooth function. The \emph{Riemannian gradient} $\nabla^g f \in \mathfrak{X}(M)$ is the unique vector field satisfying
\[
g_p(\nabla^g f(p), v) = df_p(v) \quad \text{for all } v \in T_p M.
\]
\end{definition}

In Euclidean space, this recovers the usual gradient. On a Riemannian manifold, the gradient is the tangent vector that points in the direction of steepest ascent, as measured by the metric $g$.

Given this structure, we can generalize gradient descent to Riemannian manifolds. Suppose we want to minimize a smooth function $f: M \to \mathbb{R}$. Starting from an initial point $x_0 \in M$, we iteratively update:
\[
x_{k+1} = \operatorname{Retr}_{x_k}\left(-\alpha_k \nabla^g f(x_k)\right),
\]
where $\alpha_k > 0$ is a step size, and $\operatorname{Retr}_{x_k}$ is a \emph{retraction}, a smooth map that approximates the exponential map and ensures the updated point remains on the manifold.

\begin{definition}[Retraction]
A retraction on $M$ is a smooth map $R: TM \to M$ such that:
\begin{enumerate}
\item $R_p(0_p) = p$ for all $p \in M$
\item $dR_p|_{0_p} = \id_{T_p M}$, where we identify $T_{0_p}T_p M \cong T_p M$
\end{enumerate}
\end{definition}

\section{Mathematical Proofs}
In this section, we present formal proofs for the theoretical results stated in the main text based on \cite{buchfink2024modelreductionmanifoldsdifferential}. Each proof is self-contained and follows from the assumptions and definitions introduced earlier.
\begin{theorem}[Exact reproduction of a solution]
Assume that the FOM  is uniquely solvable and consider a reduction map $R \in C^{\infty}(T\mathcal{H}, T\check{\mathcal{H}})$ for the smooth embedding $\varphi \in C^{\infty}(\check{\mathcal{H}}, \mathcal{H})$ and a parameter $\mu \in \mathcal{P}$. Assume that the ROM is uniquely solvable and $\gamma(t; \mu) \in \varphi(\check{\mathcal{H}})$ for all $t \in \mathcal{I}$. Then the ROM solution $\check{\gamma}(\cdot; \mu)$ exactly recovers the solution $\gamma(\cdot; \mu)$ of the FOM  for this parameter, i.e.,
\begin{equation}
\varphi(\check{\gamma}(t; \mu)) = \gamma(t; \mu) \quad \text{for all } t \in \mathcal{I}.
\end{equation}
\end{theorem}
\begin{proof}
Since $\gamma(t; \mu) \in \varphi(\mathcal{H})$ for all $t \in \mathcal{I}$, we can construct $\check{\beta}$ as in Equation \ref{eq: beta}. It remains to show that $\check{\beta}$ satisfies the ROM. First, we obtain
\begin{equation}
\check{\gamma}_0(\mu) = \varrho(\gamma_0(\mu)) = \varrho(\gamma(t_0; \mu)) = (\varrho \circ \varphi)\left(\check{\beta}(t_0; \mu)\right) = \check{\beta}(t_0; \mu),
\end{equation}
where the last equality is due to the projection property for the point reduction. Second, $\check{\beta}$ suffices the initial value problem of the ROM since the tangent projection property implies with Equation \ref{eq: condition_2}:
\begin{equation}
R|_{\varphi(\check{\beta}(t;\mu))}\left(X(\mu)|_{\varphi(\check{\beta}(t;\mu))}\right) = \left(R|_{\varphi(\check{\beta}(t;\mu))} \circ \mathrm{d}\varphi|_{\check{\beta}(t;\mu)}\right)\left(\frac{\mathrm{d}}{\mathrm{d}t}\check{\beta}\bigg|_{t;\mu}\right) = \frac{\mathrm{d}}{\mathrm{d}t}\check{\beta}\bigg|_{t;\mu}. 
\end{equation}
\end{proof}

\begin{theorem}[Manifold Galerkin Projection]
\label{them: 22}
\textit{Consider a smooth embedding $\varphi$ and a point reduction $\varrho$ for $\varphi$. Then, the differential of the point reduction $\varrho$ is a left inverse to the differential of the embedding $\varphi$. Consequently,}
\begin{equation}
R_{\text{MPG}}: T\mathcal{H} \to T\check{\mathcal{H}} \quad (h, v) \mapsto (\varrho(h), d\varrho|_h(v)) 
\end{equation}
\textit{is a smooth reduction map for $\varphi$, which we call the MPG reduction map for $(\varrho, \varphi)$.}
\end{theorem}
\begin{proof}
Since $\varphi$ is a homeomorphism onto its image, we know that $\varphi^{-1} : \varphi(\check{\mathcal{H}}) \to \check{\mathcal{H}}$ exists. Under mild assumptions, the extension lemma for smooth functions guarantees that we can find a smooth extension $\varrho$ of $\varphi^{-1}$ \cite[Lemma 2.26]{lee_smooth_manifolds} , which by construction satisfies the point projection property. Differentiating the point projection property with the chain rule implies
\begin{equation}
\mathrm{d}\varrho|_{\varphi(\check{h})} \circ \mathrm{d}\varphi|_{\check{h}} = \mathrm{d}(\mathrm{id}_{\check{\mathcal{H}}})|_{\check{h}} = \mathrm{id}_{T_{\check{h}}\check{\mathcal{H}}} : T_{\check{m}}\check{\mathcal{H}} \to T_{\check{h}}\check{\mathcal{H}}, \tag{3.10}
\end{equation}
i.e., $\mathrm{d}\varrho|_{\varphi(\check{h})}$ is a left-inverse to $\mathrm{d}\varphi|_{\check{h}}$. 
\end{proof}

\begin{theorem}[Generalized Manifold Galerkin (GMG)]
\label{thm: ss}
Let $\mathcal{M}$ be a manifold of dimension $N$ endowed with a non-degenerate $(0,2)$-tensor field $\tau \in \Gamma(T^{(0,2)}(T\mathcal{M}))$, and let $\varphi \in \mathcal{C}^{\infty}(\tilde{\mathcal{M}}, \mathcal{M})$ be a smooth embedding such that the reduced tensor field $\tilde{\tau} := \varphi^*\tau \in \Gamma(T^{(0,2)}(T\tilde{\mathcal{M}}))$ is nondegenerate. 

Define the vector bundle $E_{\varphi(\tilde{\mathcal{M}})} := \bigcup_{m \in \varphi(\tilde{\mathcal{M}})} T_m\mathcal{M}$ and the generalized manifold Galerkin (GMG) mapping as follows:
$$
R_{\text{GMG}}: T\mathcal{M} \supseteq E_{\varphi(\tilde{\mathcal{M}})} \to T\tilde{\mathcal{M}}, \quad (m,v) \mapsto \left(\varrho(m), \left(\sharp_{\check{\tau}} \circ \mathrm{d}\varphi^*|_{\varrho(m)} \circ b_{\tau}\right)(v)\right),
$$
where $\varrho: \varphi(\tilde{\mathcal{M}}) \to \tilde{\mathcal{M}}$ is the inverse of $\varphi$, $b_{\tau}$ is the musical isomorphism induced by $\tau$, and $\sharp_{\tau}$ is the musical isomorphism induced by $\tilde{\tau}$. Then $R_{\text{GMG}}$ is a reduction map for $\varphi$.
\end{theorem}

\begin{assumption}
\label{assumption_necessary}
Given the nondegenerate $(0,2)$-tensor field $\tau \in \Gamma(T^{(0,2)}(T\mathcal{M}))$, the smooth embedding $\varphi \in C^{\infty}(\check{\mathcal{M}}, \mathcal{M})$ is such that the reduced tensor field
\begin{equation*}
\check{\tau} := \varphi^*\tau \in \Gamma(T^{(0,2)}(T\check{\mathcal{M}})),
\end{equation*}
is nondegenerate.
\end{assumption}

\begin{lemma}
\label{lemmaa}
Under Assumption \ref{assumption_necessary}, it holds
\begin{equation}
\mathrm{d}\varphi^*|_{\check{m}} \circ b_\tau \circ \mathrm{d}\varphi|_{\check{m}} = b_{\check{\tau}} \in C^{\infty}(T_{\check{m}}\check{\mathcal{M}}, T_{\check{m}}^*\check{\mathcal{M}}). 
\end{equation}
\end{lemma}

\begin{proof}
     We prove the statement in index notation. We obtain for all $\check{m} \in \check{\mathcal{M}}$, all $\check{v} \in T_{\check{m}}\check{\mathcal{M}}$ and all $1 \leq i \leq n$
\begin{align}
(b_{\check{\tau}}(\check{v}))_i &= (\check{\tau}|_{\check{m}})_{ij} \check{v}^j = (\tau|_{\varphi(\check{m})})_{\ell_1 \ell_2} \frac{\partial\varphi^{\ell_1}}{\partial x^i}\bigg|_{\check{m}} \frac{\partial\varphi^{\ell_2}}{\partial x^j}\bigg|_{\check{m}} \check{v}^j \\
&= \frac{\partial\varphi^{\ell_1}}{\partial x^i}\bigg|_{\check{m}} (\tau|_{\varphi(\check{m})})_{\ell_1 \ell_2} \frac{\partial\varphi^{\ell_2}}{\partial x^j}\bigg|_{\check{m}} \check{v}^j = \left((\mathrm{d}\varphi^*|_{\check{m}} \circ b_\tau \circ \mathrm{d}\varphi|_{\check{m}})(\check{v})\right)_i. \quad
\end{align}
\end{proof}

Now, let us prove Theorem \ref{thm: ss}.

\begin{proof}
    By construction, Lemma \ref{lemmaa}, and musical isomorphism pairing, and point projection property we obtain,
\begin{equation*}
R_{\text{GMG}}|_{\varphi(\check{m})} \circ \mathrm{d}\varphi|_{\check{m}} = \sharp_{\check{\tau}} \circ \mathrm{d}\varphi^*|_{(\varrho \circ \varphi)(\check{m})} \circ b_\tau \circ \mathrm{d}\varphi|_{\check{m}} = \sharp_{\check{\tau}} \circ b_{\check{\tau}} = \mathrm{id}_{T_{\check{m}}\check{\mathcal{M}}},
\end{equation*}
\end{proof}

\begin{lemma}
\textit{Consider a symplectic manifold $(\mathcal{H}, \omega)$, a smooth manifold $\check{\mathcal{H}}$, and a smooth embedding $\varphi \in C^{\infty}(\check{\mathcal{H}}, \mathcal{H})$ such that $\check{\omega} := \varphi^* \omega$ is nondegenerate. Then $\check{\omega}$ is a symplectic form, $(\check{\mathcal{H}}, \check{\omega})$ is a symplectic manifold, and $\varphi$ is a symplectomorphism.}
\end{lemma}
\begin{proof}
It is sufficient to show that $\check{\omega} = \varphi^*\omega$ is a symplectic form, which in this case results in showing that $\check{\omega}$ is skew-symmetric and closed. The skew-symmetry is inherited for all points $\check{h} \in \check{\mathcal{H}}$ since we have:
\begin{equation*}
(\check{\omega}|_{\check{h}})_{j_1 j_2} = (\omega|_{\varphi(\check{h})})_{\ell_1 \ell_2} \frac{\partial\varphi^{\ell_1}}{\partial \check{x}^{j_1}}\bigg|_{\check{h}} \frac{\partial\varphi^{\ell_2}}{\partial \check{x}^{j_2}}\bigg|_{\check{h}} = -(\omega|_{\varphi(\check{m})})_{\ell_2 \ell_1} \frac{\partial\varphi^{\ell_2}}{\partial \check{x}^{j_2}}\bigg|_{\check{h}} \frac{\partial\varphi^{\ell_1}}{\partial \check{x}^{j_1}}\bigg|_{\check{h}} = -(\check{\omega}|_{\check{h}})_{j_2 j_1}.
\end{equation*}

Closedness is inherited since the pullback of a closed form is closed. 
\end{proof}

\begin{theorem}
\label{thm: SMG}
\textit{The SMG-ROM is a Hamiltonian system $(\check{\mathcal{H}}, \tilde{\omega}, \check{H})$ with the reduced Hamiltonian $\check{H} := \varphi^* H = H\circ \varphi$.}    
\end{theorem}
\begin{proof}
The ROM vector field with the SMG reduction reads with (a) point projection property and musical isomorphism pairing, and (b) chain rule
\begin{align}
R_{\text{SMG}}|_{\varphi(\check{h})}\left(X_\mathcal{H}|_{\varphi(\check{h})}\right) &= \left(\sharp_{\check{\omega}} \circ \mathrm{d}\varphi^*|_{(\varrho \circ \varphi)(\check{h})} \circ b_\omega\right)\left(\sharp_\omega\left(\mathrm{d}\mathcal{H}|_{\varphi(\check{h})}\right)\right) \\
&\stackrel{(a)}{=} \sharp_{\check{\omega}}\left(\mathrm{d}\varphi^*|_{\check{m}}\left(\mathrm{d}\mathcal{H}|_{\varphi(\check{h})}\right)\right) \stackrel{(b)}{=} \sharp_{\check{\omega}}\left(\mathrm{d}\check{\mathcal{H}}|_{\check{h}}\right), \nonumber
\end{align}
which is exactly the Hamiltonian vector field of the Hamiltonian system $(\check{\mathcal{M}}, \check{\omega}, \check{\mathcal{H}})$. 

The reduced vector field in the SMG-ROM is then simplified as follows:
\begin{equation}
\boldsymbol{R}_{\text{SMG}}|_{\varphi(\check{m})}\left(\boldsymbol{X}_\mathcal{H}|_{\varphi(\check{h})}\right) = \left(D\varphi|_{\check{h}}^{\top} \omega|_{\varphi(\check{h})} D\varphi|_{\check{h}}\right)^{-1} D\varphi|_{\check{h}}^{\top} \omega|_{\varphi(\check{h})} \boldsymbol{X}_\mathcal{H}|_{\varphi(\check{h})} \in \mathbb{R}^n. 
\end{equation}
\begin{equation*}
\underbrace{D\varphi|_{\check{h}}^{\top} \omega|_{\varphi(\check{h})} \boldsymbol{X}_\mathcal{H}|_{\varphi(\check{m})}}_{= D\varphi|_{\check{h}}^{\top} D\mathcal{H}|_{\varphi(\check{h})} = D\check{\mathcal{H}}|_{\check{h}}}
\end{equation*}
\end{proof}

\section{Additional Numerical Results}

In this section, we provide additional numerical results comparing the training times of various neural network models across multiple physical systems to assess their computational efficiency.

Figure \ref{fig:computational_efficiency_comparison} illustrates the training time results across four physical systems, and shows a clear efficiency ranking among the models. The standard MLP is the fastest, taking 36–101 seconds, and serves as the baseline. HNN and DoubleHeadHNN are about 2–3 times slower (94–248 seconds), with DoubleHeadHNN sometimes faster than HNN despite its extra complexity. CholeskyHNN has similar training times, between 105 and 273 seconds. GeoHNN is the slowest, running 3–12 times longer than MLP (242–1030 seconds), and scales poorly with system complexity, taking 17 minutes to train on the 3 coupled oscillators system, compared to 1.4 minutes for MLP. This highlights a trade-off: adding geometric constraints improves physical fidelity but increases computational cost, especially for Riemannian-based methods like GeoHNN.

\begin{figure}[H]
    \centering
    \begin{minipage}[b]{0.4\textwidth}
        \centering
        \includegraphics[width=\textwidth]{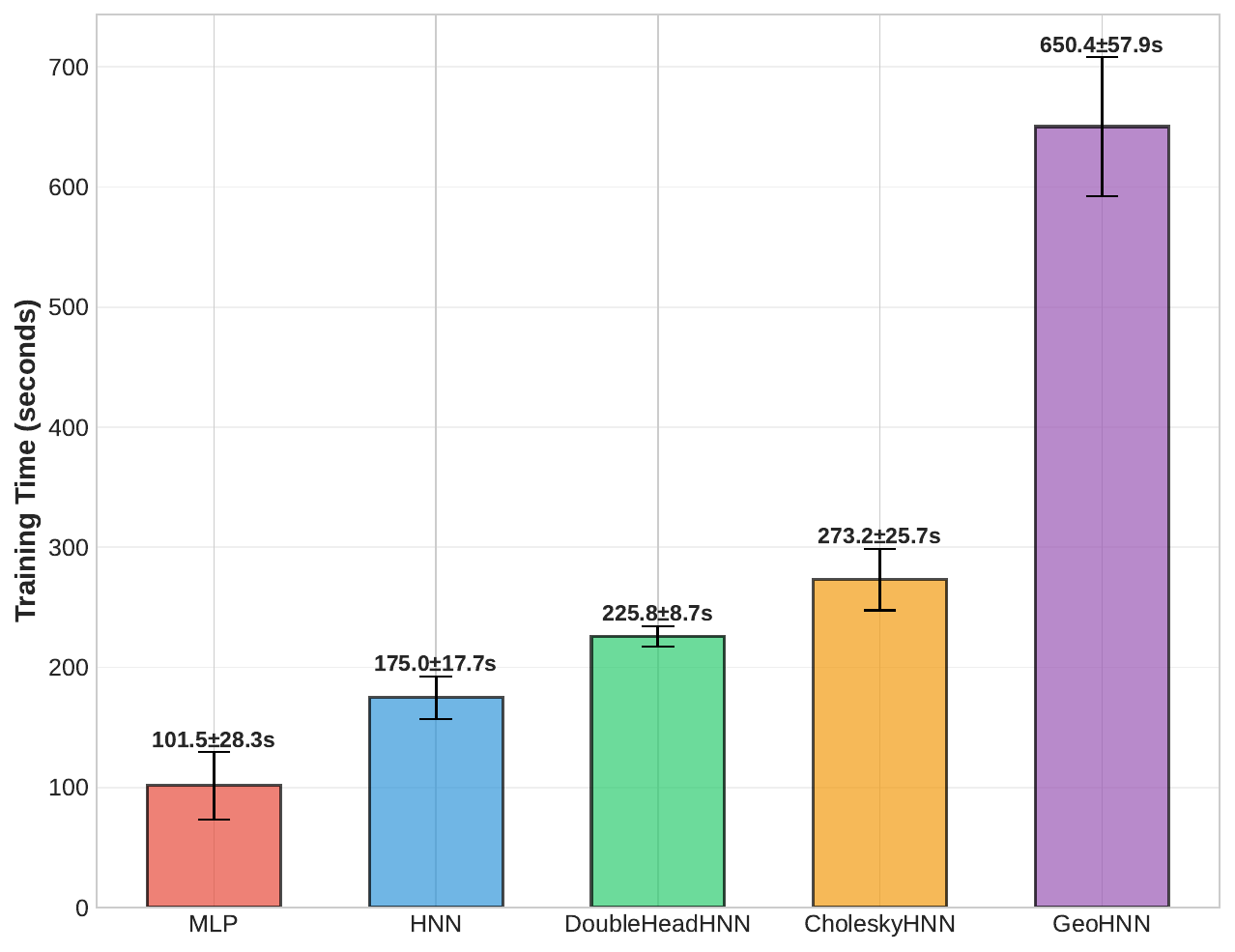}
        \\ (a) Simple pendulum system
    \end{minipage}
    \hfill
    \begin{minipage}[b]{0.4\textwidth}
        \centering
        \includegraphics[width=\textwidth]{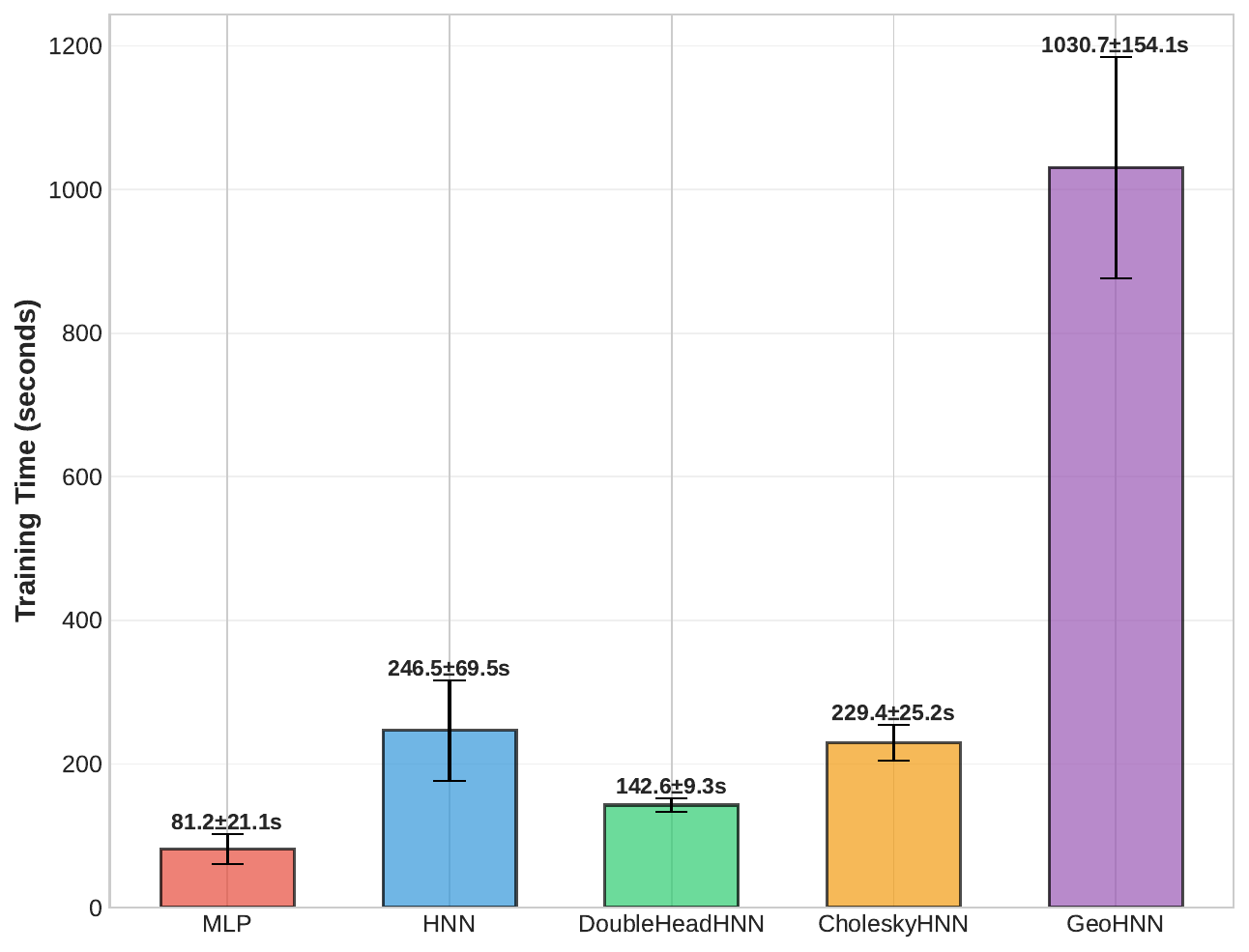}
        \\ (b) 3 coupled oscillators system
    \end{minipage}
    
    \vspace{1em}
    
    \begin{minipage}[b]{0.4\textwidth}
        \centering
        \includegraphics[width=\textwidth]{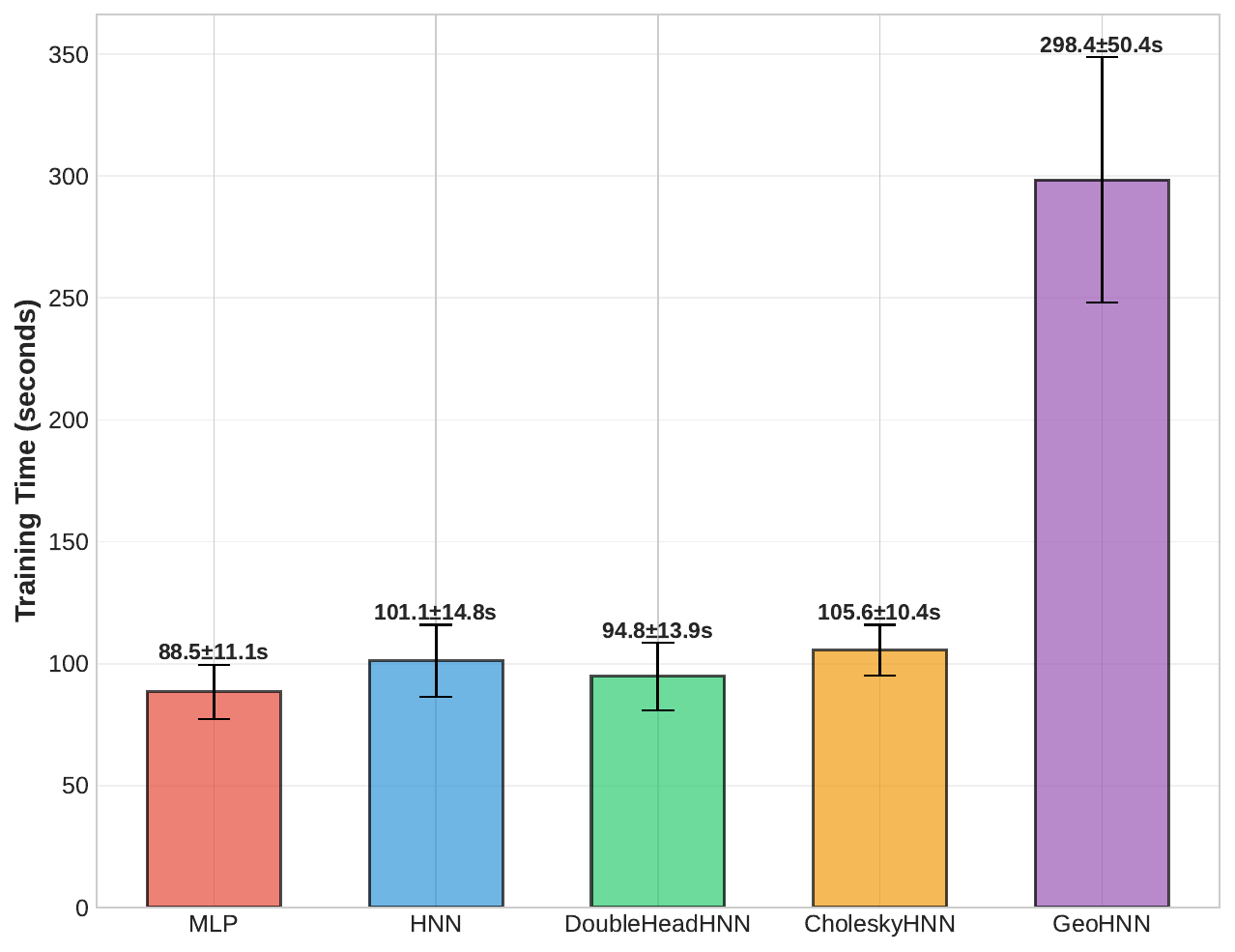}
        \\ (c) Mass spring system
    \end{minipage}
    \hfill
    \begin{minipage}[b]{0.4\textwidth}
        \centering
        \includegraphics[width=\textwidth]{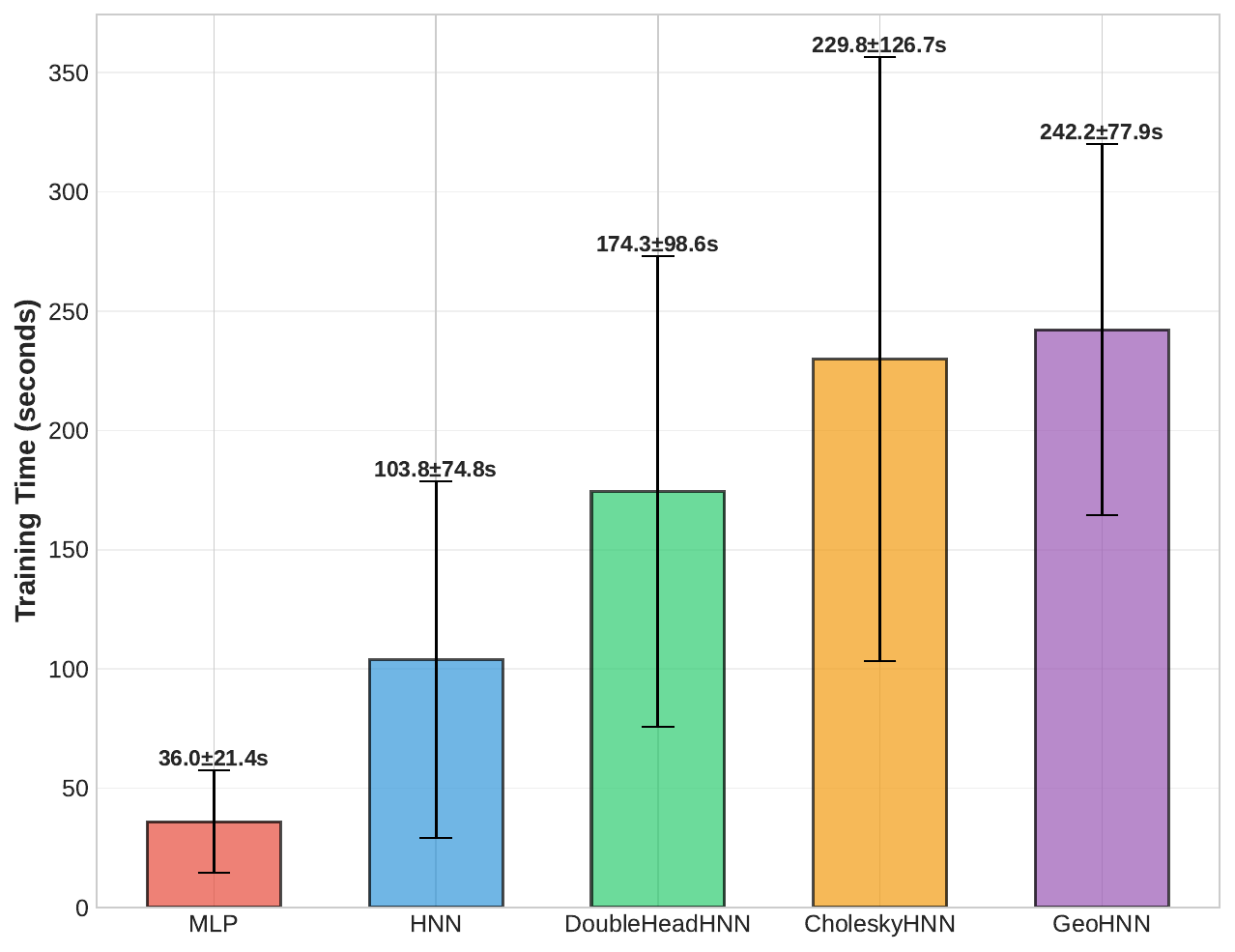}
        \\ (d) Two body problem
    \end{minipage}
    
    \caption{Training time comparison for different models across various physical systems.}
    \label{fig:computational_efficiency_comparison}
\end{figure}
\end{document}